\documentclass[runningheads]{llncs}
\usepackage[T1]{fontenc}
\usepackage[utf8]{inputenc}
\usepackage{amsmath}
\usepackage{amssymb}
\usepackage{amsfonts}
\usepackage{algorithm}
\usepackage{algpseudocode}
\usepackage{graphicx}
\usepackage{listings}
\usepackage{subcaption}
\usepackage{mathtools}
\usepackage{textcomp}
\usepackage{xcolor}
\usepackage{bbm}
\usepackage{mathtools}
\usepackage{caption}
\usepackage{subcaption}
\usepackage{diagbox}
\usepackage{array}
\usepackage{mathtools}
\usepackage{breqn}
\usepackage{enumitem}
\usepackage{physics}
\usepackage{thm-restate}
\usepackage{url}
\usepackage{tikz-cd}
\usepackage[normalem]{ulem}
\usepackage{comment}
\usepackage{cite}

\newcommand{\defeq}{\vcentcolon=}

\DeclareMathOperator*{\argmin}{arg\,min}

\usepackage{booktabs}
\allowdisplaybreaks

\begin{document}
\title{Group privacy for personalized federated learning}
%
%

\author{Filippo Galli\inst{1} \and
Sayan Biswas\inst{2,3} \and
Kangsoo Jung\inst{3} \and
Tommaso Cucinotta\inst{4} \and
Catuscia Palamidessi\inst{2,3}}
\authorrunning{F. Galli et al.}
%
\institute{Scuola Normale Superiore, Pisa, Italy \and INRIA, Palaiseau, France \and
LIX, \'Ecole Polytechnique, Palaiseau, France
\and 
Scuola Superiore Sant'Anna, Pisa, Italy\\
\email{filippo.galli@sns.it, 
\{sayan.biswas, gangsoo.zeong\}@inria.fr\\
tommaso.cucinotta@santannapisa.it,
catuscia@lix.polytechnique.fr}}
\maketitle              
\begin{abstract}
Federated learning (FL) is a particular type of collaborative machine learning, where participating peers/clients process their data locally, sharing only updates to the collaborative model. This enables to build privacy-aware distributed machine learning models, among others. The goal is the optimization of a statistical model's parameters by minimizing a cost function of a collection of datasets which
are stored locally by a set of clients. This process exposes the clients to two issues: leakage of private information and lack of personalization of the model. On the other hand, with the recent advancements in various techniques to analyze and handle data, there is a surge of concern for the privacy violation of the participating clients. To mitigate this, differential privacy and its variants serve as a standard for providing formal privacy guarantees. Often the clients represent very heterogeneous communities and hold data which are very diverse. Therefore, aligned with the recent focus of the FL community to build a framework of personalized models for the users representing their diversity, it is also of utmost importance to protect the clients' sensitive and personal information against potential threats. To address this goal we consider $d$-privacy, also known as metric privacy, which is a variant of local differential privacy, 
using a  a metric-based obfuscation technique that preserves the topological distribution of the original data. To cope with the issue of protecting the privacy of the clients and allowing for personalized model training to enhance the fairness and utility of the system, we propose a method to provide group privacy guarantees exploiting some key properties of $d$-privacy which enables personalized models under the framework  of FL. We provide with theoretical justifications to the applicability and experimental validation on real-world datasets to illustrate the working of the proposed method.

\end{abstract}
%
%
%
\section{Introduction}\label{section:intro}

With the recent advancements in technology, there has been a significant surge in the value and need of data to perform various kinds of statistical analyses by academia and industry alike. In particular, with modern development in machine learning and data science, the requirement of collecting massive datasets from users, often containing their sensitive personal information, is becoming more and more popular. Their functionality varies from descriptive queries to training large machine learning models with millions of parameters for more complex tasks in hand.

There are multiple advantages to having access to all the necessary
data in a single location, mostly related to efficiency: faster computation, reduced communication costs between the computing and storage nodes, and, in general, a more direct control over the population of data points. However, alongside this massive rise in need to collect and store data, the risks of violation of the users' privacy are becoming more and more significant and concerning~\cite{nist,lemetayer:hal-01420983}. Of late, users are increasingly concerned about the use, retention, access,
and a potential involuntary disclosure of their private information. Federated learning (FL)~\cite{fed-l-0} is a collaborative machine learning paradigm where the devices of the users serve not only for data harvesting but they are also directly involved in
training a global predictive model without ever sending the raw data to a central server, taking rudimentary a step towards the goal of protecting the users' privacy. 

In this context, the central server orchestrates rounds of
parameter fitting by selecting a random subset of users and sending them
a model for local optimization. Following this, the users optimize their model parameters to minimize a loss function over their local data and communicate back to the central server the updated
model, typically computed using gradient-based methods. The server aggregates the updates received from the participating users to the global model and, thus, a new round can commence with a new subset of users. The process is repeated until convergence, i.e. until there is no substantial decrease in the loss function round over round.

Nonetheless, avoiding the release of user's raw data only provides a lax protection to potential attacks violating the users' privacy~\cite{hitaj2017deep, nasr2019comprehensive, deep-leakage}, as it falls in the pitfall of ``only releasing summary
statistics''~\cite{foundations-of-dp}, which is the set of updated model parameters transmitted
to the central server.

One of the most successful approaches to address this issue in a rather robust way is along the lines of \emph{Differential Privacy} (DP)~\cite{DworkDP1,DworkDP2}, which mathematically guarantees that a query output for a dataset does not change significantly regardless of whether a specific personal record is contained or not. For instance, a model trained for next-word prediction under this framework will not make suggestions that may potentially leak a user's private data.

However, the classical central version of DP requires a trusted curator who is responsible for adding noise to the data before publishing or performing any kind of analytics on it. A major drawback of such a central model is that it is vulnerable to security breaches via a single point of failure because of its over-dependency on the central server for storage. Moreover, there is the risk of having an adversarial curator. To circumvent the need of such a central trusted server, a local model of DP a.k.a. \emph{local differential privacy} (LDP)~\cite{DuchiLDP} has been in the spotlight recently where the users locally perturb their personal data using LDP mechanisms (e.g. $k$-randomized response~\cite{kairouz2016discrete}) before communicating them to the server.

One of the recently popularized standards in location privacy is geo-indistinguishability ~\cite{broadening}, which optimizes the quality of service (QoS) of the users while preserving a generalized notion of LDP on their location data. The obfuscation mechanism of geo-indistinguishability depends on the \emph{Euclidean} distance between the original location of users and a potential noisy location reported by them~\cite{Bordenabe:14:CCS,Fernandes:21:LICS}. This metric-based generalized variant of LDP, when used beyond the scope of location data under Euclidean distance, is known as \emph{$d$-privacy} by the community, which is applicable to any form of datasets under any notion of distance. $d$-privacy can be implemented directly on users' devices (tablet, smartphone, etc.), where users can explicitly control their desired privacy-protection level, while preserving the spatial distribution of their data due to the metric-based obfuscation mechanism. This makes $d$-privacy very appealing. 

In the context of FL, LDP mechanisms obfuscate the local updates to the model released during each round, so that information coming from any user is indistinguishable up to a certain factor. In the trade-off between privacy and accuracy, both central and local paradigms of DP may reduce the overall accuracy of the converged model because of the randomization of the information released by users.  

In general, the central and local models of DP require bounded sensitivity of the query function or a bounded domain for the data coming from the users, respectively. Since the
model parameter vectors are not necessarily bounded a-priori, it is important to acknowledge also another important source of error, which comes from \emph{clipping}~\cite{adaptive} the domain of the information released by the users.
Forcibly truncating the updates to the model leads to a clipped distribution of the parameter vectors as seen by the server and, therefore, the aggregation step of the optimization process is, often, biased \cite{suriyakumar2021chasing}. In particular, the fairness of the model becomes questionable when the \emph{minorities} of the dataset do not get represented after truncation, making the aggregated model
biased against the users harboring the unrepresented data after clipping~\cite{li2020federated, suriyakumar2021chasing}.  This is particularly problematic when the empirical distribution is used for model personalization in FL. Since clustering of the parameter vectors needs to be performed on the sanitized values reported by the users, it is of utmost importance for the server to receive the \emph{unclipped} distribution to engender a notion of fairness in the model.

Learning personalized models is a way to address the problem of heterogeneity in the data distribution of the users involved in the federated optimization. If it is possible to assume that groups of users have local datasets sampled from the same underlying probability distribution, it may be beneficial to optimize federated models by aggregating only the information coming from users of the same group. This can be particularly useful in a number of applications, e.g.: i) natural language processing models trained on datasets with regional dialects and localized language variations; ii) recommender systems for news article suggestions based on the political affiliation of users; iii) facial expression recognition with ethnically diverse dataset members. 

To address the above-mentioned issues, we investigate the possibility of providing local privacy guarantees and allow for personalized models in FL. Thus, we propose the adoption of $d$-privacy mechanisms to obfuscate the information released locally by each user in the federated training, and define an algorithm for personalized federated learning that takes advantage of distance metric-based privacy guarantees for clustering participating users.

More precisely, our key contributions in this paper are outlined as follows:
\begin{enumerate}
    \item We provide an algorithm for the collaborative training of machine learning models, which builds on top of state-of-the-art strategies for model personalization.
    \item We formalize the privacy guarantees in terms of $d$-privacy. To the best of our knowledge, this is the first time that $d$-privacy is used in the context of machine learning. 
    \item We study the Laplace mechanism on high dimensions, under Euclidean distance, based on a generalization of the Laplace distribution in $\mathbb{R}$, and we give a closed form expression.
    \item We provide an efficient procedure for sampling from such distribution. 

\end{enumerate}

The rest of the paper is organized as follows. Section~\ref{section:Background} introduces fundamental notions for federated learning and differential privacy. Section~\ref{section:Related} discusses related work. Section~\ref{section:Algorithm} explains the proposed algorithm for personalized federated learning with group privacy. 
Section~\ref{section:Experiments} validates the proposed procedure through experimental results.
Section~\ref{section:Conclusion} concludes and discusses future work.

\section{Background}\label{section:Background}

\begin{table}[htbp]
\caption{Table of Notations}\label{table:notations}
\centering
\resizebox{\columnwidth}{!}{%
\begin{tabular}{|c|c|}
\hline
Notation & Description \\ \hline
 
$\mathcal{X}$  & Domain of original values \\
$d(.)$ & Distance metric  on $\mathcal{X}$\\
$\mathcal{Y}$  & Domain of secrets \\
$\mathbb{P}_{\mathcal{K}}\left[y\vert x\right]$ & Prob. that mechanism $\mathcal{K}$ reports $x\in\mathcal{X}$ as $y\in\mathcal{Y}$\\
$N$ & Total number of clients\\
$\mathbb{D}$ & Domain of the data points held by the users\\
$k$ & Number of clusters, hypotheses and distributions\\
$n$ &Number of model parameters\\
$f(.)$ & $f\colon \mathbb{R}^n\times \mathbb{D}\mapsto \mathbb{R}_{\geq 0}$; Cost function\\
$\mathcal{D}_j$ & Probability density function of the $j^{\text{th}}$ distribution\\
$Z_c$ & Collection of data points held by client $c$\\
$S_j^*$ & Subset of clients whose data is sampled from $\mathcal{D}_j$\\
$S_j$ & Estimate of $S_j^*$\\
$\theta_j^*$ & Minimizer of $F(\theta_j)$\\
$\theta_j$ & Parameter vector\\
$\tilde{\theta}_j^*$ & Estimate of $\theta_j^*$\\
$F(\theta_j)$ & Expectation of $f(.)$ over $ z \sim \mathcal{D}_j$\\ 
$\tilde{F}(\theta_j)$ & Empirical estimate of $F(\theta_j)$\\
$\tilde{F_c}(\theta_j;Z_c)$ & $\tilde{F}(\theta_j)$ evaluated on client's $c$ data points\\
$\hat{\theta}_{j, c}^{(t)}$ & Sanitized and updated $j^{\text{th}}$ parameter vector released by $c$\\
$\mathcal{L}_{\varepsilon}$ & $\mathcal{L}_{\varepsilon}\colon \mathbb{R}^n\mapsto \mathbb{R}^n$; Laplace mechanism providing $\varepsilon$-$d$-privacy\\ 
$\mathcal{L}_{x_0,\varepsilon}(x)$ & $Ke^{-\varepsilon d(x,x_0)}$ where $K = \frac{\varepsilon^n\Gamma(\frac{n}{2})}{2 \pi^{\frac{n}{2}}\Gamma(n)}$\\
$\gamma_{\varepsilon, n}(r)$ & Gamma distribution with shape $n$ and rate $\varepsilon$.\\
$\mathbb{S}_n(r)$ & Surface of the sphere in $\mathbb{R}^n$ of radius $r$\\
$\mathcal{G}_n(0, \sigma^2)$ & Gaussian distribution in $\mathbb{R}^n$\\
$\Gamma(.)$ & Gamma function\\
$\nu$ &  Noise multiplier\\
${}_{1}x_n$ & A unit vector in $\mathbb{R}^n$\\
$\Delta$ & A generic random vector\\
\hline
\end{tabular}%
}
\end{table}

\subsection{Federated learning and personalization} \label{flp}
Collaborative learning with privacy and communication constraints has
received much attention since the introduction of federated learning
\cite{fed-l-0, fed-l-1, fed-l-2, fed-l-3},
which aims to train a global machine learning model on a distributed
collection of non-i.i.d. datasets stored on devices whose raw data cannot
be disclosed. Focusing on the personalized federated learning setting, we
adopt the notation of \cite{ghosh} to
cast the problem in the framework of stochastic optimization and find the
set of  minimizers $\theta_j^*$ with $j \in \left\{ 1, \dots, k \right\}$ of the cost functions
\begin{equation} \label{erm:1}
F(\theta_j) = \mathbb{E}_{z\sim\mathcal{D}_j} \left[f(\theta_j;z)\right],
\end{equation}
where $\mathcal{D}_j$ is the data distribution which can only be accessed through a collection of datasets $Z_c=\left\{z_i | z_i \sim \mathcal{D}_j, z_i \in \mathbb{D} \right\}$ with $c \in C = \left\{ 1, \dots, N \right\}$, the set of clients. $C$ is partitioned in $k$ disjoint sets 
\begin{equation}\label{erm:2}
S_j^* = \{c \mid \forall z \in Z_c, \, z \sim \mathcal{D}_j\}\,\forall\,j\in[k]
\end{equation}
The mapping $c \rightarrow j$ is
unknown and we rely on estimates $S_j$ of the membership of $Z_c$ to compute
the empirical cost functions 
\begin{equation} \label{erm:3}
\tilde{F}(\theta_j) = \frac{1}{|S_j|}\sum_{c \in S_j} \tilde{F_c}(\theta_j;Z_c)
\end{equation}
with 
\begin{equation} \label{erm:4}
\tilde{F_c}(\theta_j;Z_c) = \frac{1}{|Z_c|}\sum_{z_i \in Z_c}f(\theta; z_i)
\end{equation}

The cost function $f \colon \mathbb{R}^n\times \mathbb{D} \mapsto \mathbb{R}_{\geq 0}$ is applied on $z \in \mathbb{D}$, parametrized by the vector $\theta_j \in \mathbb{R}^n$. Thus, the optimization aims to find, $\forall \,j\,\in[k]$,
\begin{equation} \label{erm:5}
\tilde{\theta}_j^* = \argmin_{\theta_j}\tilde{F}(\theta_j)
\end{equation}

A summary of the main notational elements used throughout the paper can be found in \ref{table:notations}.

\subsection{Differential privacy and machine learning}

\emph{Differential privacy (DP)}~\cite{DworkDP1,DworkDP2},  introduced as a property of
queries of statistical databases to measure information leakage, is the state-of-the-art approach to formalize privacy guarantees by mathematically ensuring that an output of a given query probabilistically does not alter irrespective of whether a specific record is contained in it or not.

\begin{definition}[Differential privacy~\cite{DworkDP1,DworkDP2}]\label{def:DP}
A mechanism $\mathcal{M}$ is $(\varepsilon,
\delta)$-differentially private if for all adjacent
databases~\footnote{Databases are said to be adjacent or neighbors when they differ in one record.} $D,\,D'$
and for every measurable $S\,\subseteq\,\operatorname{Range}(\mathcal{M})$
holds that:
\begin{equation} \label{eq:cdp}
  \mathbb{P}\left[\mathcal{M}(D)\,\in\,S\,\right] \leq
  e^{\varepsilon}\mathbb{P}\left[\mathcal{M}(D')\,\in\,S\right] + \delta
\end{equation}
\end{definition}

To mitigate the major drawback of central model of DP that requires a trusted central dependancy from the server, a local variant of the central model has been studied recently by the community and termed as local differential privacy (LDP)~\cite{DuchiLDP}, where the users locally obfuscate their data and send the noisy data to the server such that a particular entry of a user's data probabilistically does not have an impact on the outcome of the query. 

\begin{definition}[Local differential privacy~\cite{DuchiLDP}]
\label{def:ldp}
Let $\mathcal{X}$ and $\mathcal{Y}$ denote the spaces of the original and the perturbed noisy data, respectively~\footnote{Usually in LDP $\mathcal{X}$ and $\mathcal{Y}$ are discrete domains but for the sake of uniformity with the other definitions we extend LDP to continuous domains.}. A mechanism $\mathcal{M}$ provides \emph{$(\varepsilon,\delta)$-local differential privacy} if, for all $x,\,x'\,\in\,\mathcal{X}$, and all measurable $S\,\subseteq\,\mathcal{Y}$, we have:

\begin{equation}\label{eq:ldp}
\mathbb{P}\left[\mathcal{M}(x) \in S\right] \leq e^{\varepsilon}\mathbb{P}\left[\mathcal{M}(x') \in S\right]+\delta
\end{equation}

\end{definition}

The local model for differential privacy~\cite{what-can-we} can be derived from \eqref{eq:cdp} when $x,\,x'$ are taken to be datasets of only one record. Therefore LDP is a stronger condition as it requires the mechanism to satisfy DP for any two values of the domain of
data $\mathcal{X}$. 

There are different approaches studied in the literature that apply DP in machine learning~\cite{shokri, abadi, dprnn}, but, possibly, one of the most successful lines of work is based on evaluating how much each user, participating in the training dataset, has contributed to the trained model. Essentially, gradient-based optimization of a machine learning model, parametrized by $\theta$, works by computing the gradient of a loss
function $\nabla f(\theta, z)$ with respect to $\theta$, for a number of iterations, evaluated over a batch of $z$, and updating the parameters according to the (stochastic) gradient descent algorithm~\cite{bottou2012stochastic}. If $\norm{\nabla f(\theta, z)}_2$ is clipped to a value $g_{\max}$, then the function querying the dataset has bounded sensitivity and, thus, the Gaussian mechanism with the properties described in \cite{abadi} can be applied to sanitize the queries to a user's data point $z$.

In the context of FL, the procedure described in
\cite{dprnn, adaptive} requires the clients to perform a few iterations of gradient descent over their local datasets $Z_c$ and only report the difference in the parameter vector before and after the update, clipped in norm to a value $g_{\max}$, to the
central server. The server then applies the Gaussian
mechanism to compute sanitized average updates to the model parameters, thus preserving DP with a preferred privacy level.

\subsection{$d$-privacy}\label{sec:dprivacy}
$d$-privacy~\cite{broadening} is a generalization of DP for any domain $\mathcal{X}$, representing the space of original data, endowed with a distance measure $d\colon \mathcal{X}^2\mapsto \mathbb{R}_{\geq 0}$, and any space of secrets $\mathcal{Y}$. A random mechanism $\mathcal{R}: \mathcal{X} \mapsto \mathcal{Y}$ is called $\varepsilon$ $d$-private if for all $x_1,\,x_2\in\,\mathcal{X}$ and measurable $S\,\subseteq \,\mathcal{Y}$: 
\begin{equation}\label{eq:dprivacy}
  \mathbb{P}\left[\mathcal{R}(x_1) \in S\right] \leq 
  e^{\varepsilon d(x_1,x_2)} \mathbb{P}\left[\mathcal{R}(x_2) \in S\right]
\end{equation}

Note that when $x_1, x_2$ are elements of the domain of databases, and $d$ is the distance on the Hamming graph of their adjacency relation, then \eqref{eq:cdp} and \eqref{eq:dprivacy} are equivalent, reducing the applicability of $d$-privacy to that of DP. It is also worthy to note that, in general, $\mathcal{X}$ and $\mathcal{Y}$ may be different. However, in the context of this work we have the space of original data and the space of secrets to be the same, i.e, $\mathcal{X}=\mathcal{Y}$.

This notion of distance metric-based privacy has been found particularly effective in the context of location privacy~\cite{broadening,geo}, where $\mathcal{X}=\mathbb{R}^2$ and $d$ is the Euclidean distance. The authors show how the formal privacy guarantees degrade gracefully with the distance between two points, which is especially beneficial when the service provider or the server is interested in an approximate value of the true location of the users, thus striking a balance between the privacy level required by the users and the statistical accuracy of their reported values.

This approach differs from that of DP, preferable only when an aggregated information is required. To sanitize the values in $\mathcal{X}$, \cite{broadening} introduces a
generalized Laplace mechanism, although an analytical form of the probability distribution for or the sampling procedure from a domain in $\mathbb{R}^n$, for $n>2$, has not been presented. It is worth noting that the clients may decide the standard deviation of the noise they choose to inject to their real data based on a radius within which they want to be indistinguishable. For instance, providing a sanitized location with a noise of standard deviation in the order of $1$ km may be sufficient for a user
to report her rough location to query for suggestions on nearby restaurants to a service provider, and at the same time concealing her exact coordinates.

\section{Related works}\label{section:Related}

With the generalized Federated Averaging algorithm~\cite{fed-l-1,gen-fed-avg} to solve the empirical risk minimization problem in Equation \eqref{erm:5}, an aggregated global model is optimized iteratively by a series of communications between a central server and a subset of clients where the local datasets reside. In each round, the server communicates the current state of the global model and the participating clients run a number of local optimization steps before communicating back to the server the updated model or the differential update. This approach has shown to be under-performing when the local datasets are samples of non-congruent distributions, failing to minimize both the local
and global objectives at the same time.

The need for personalized federated learning, therefore, emerged as a means to address this issue, with many different techniques being proposed. In \cite{mansour}, the authors suggest three methods for
personalization based on clustering, model interpolation, and data
interpolation. The idea of hypothesis-based clustering is also studied in \cite{ghosh}, which further provides convergence guarantees of the population loss function. Clustering participating clients to give rise to a personalized model is also the approach taken in \cite{sattler}, which goes on to introduce a meta-algorithm to determine whether the clients belong to non-congruent distributions, whether the federated optimization has reached minimums of both the clients and server objectives, and a  method for clustering based on cosine similarity of the updates.

In the works introduced above, the claims of privacy protection derive from the local raw data of the clients not being disclosed throughout the communication rounds between the server and the clients. As discussed in \cite{foundations-of-dp}, disclosing any answer to a deterministic query can release private information and relying on the ``release of summary statistics'' argument (i.e. releasing only model updates instead of releasing clients' raw data) can have dramatic effects on privacy of individuals.

To confront this issue, a number of works have focused on the privatization of the (federated) optimization algorithm under the framework of DP~\cite{abadi, flcdp, dprnn, adaptive}, thus
providing formal guarantees that the learned model will not depend too much on the presence or absence of a particular user's record in the dataset used in the federated optimization. The model of the attacker is, thus, reduced to an honest but curious adversary who only has access to the trained model \cite{abadi, adaptive, dprnn}. However, in this setting, no protection is ensured against the server and any possible man-in-the-middle attacker between the clients and the server who might access the clients' updates. This has been shown to be problematic as a malicious adversary with only access to the model updates sent by the clients has enough information to reconstruct samples from the local datasets~\cite{deep-leakage}. In \cite{secure-aggregation}, the authors addressed this concern of communicating non-privatized updates to a central server by introducing a cryptographically secure aggregation protocol for the central server to compute the updated global model state from the encrypted client's updates, but at the cost of increased communication and computation requirements for both the clients and the server.

Since various kinds of communication constraints form some of the most defining characteristics of the FL setting, other works examined, instead, the use of local differential privacy mechanisms
for protection against any strong adversary that may have access to the clients' updates \cite{ldpfl, zhao2020local}. One such example is \cite{ldpfl} which obfuscates each parameter within a certain adaptively-defined range of values and adopts a parameter shuffling mechanism to amplify the privacy guarantees being motivated by the shuffle model of DP~\cite{bittau2017prochlo}, which has been extensively studied of late in the literature~\cite{sommer2019privacy,cheu2019distributed,cheu2021differentially,erlingsson2019amplification,erlingsson2020encode,balle2019privacy,balle2020private,meehan2021shuffling,koskela2021tightdiscrete,kairouz2016discrete,koskela2021tight,feldman2020hiding}. It must be noted that the mechanism in 
\cite{ldpfl} requires each parameter of the local model to be uploaded to the server one at a time, which can drastically increase the wall-clock convergence time of the algorithm when used to train modern machine learning models which easily require millions of parameters. 

In \cite{shuffled} and \cite{erlingsson2020encode} the authors adopt the framework of local differential privacy and exploit shuffling, subsampling and other techniques to amplify the guarantees in terms of central differential privacy. Notably, these techniques still rely on a trusted aggregator. 
Work \cite{cpsgd} examines quantization techniques used for improving communication efficiency to establish local differential privacy guarantees against an untrusted or negligent aggregator. Relatively to the works just mentioned, we highlight how the use of local differential privacy with non-trivial guarantees would be problematic with personalization, as, by definition, client updates belonging to the bounded domain of diameter $2\cdot g_{\max}$ should be indistinguishable up to a small multiplicative factor. In \cite{huetal} the authors address the problem of personalized and locally differentially private federated learning, but for the simple case of convex, $1$-Lipschitz cost functions of the inputs. Note that this assumption is unrealistic in most machine learning model, and exclude many statistical modeling techniques, notably neural networks. Conversely, we do not make these assumptions. 

In Table \ref{tab:comp} is provided a qualitative comparison of this effort compared with the most relevant prior work on the subject, in order to provide context of the problem and hand and its proposed solution.

To the best of our knowledge, our paper is the first work trying to optimize over the two dimensions of indistinguishability and personalization in the context of the federated learning.

\begin{table}[]
\centering
\begin{tabular}{ccccc}
\multicolumn{1}{l}{} & \multicolumn{1}{l}{\cite{dprnn}} & \multicolumn{1}{l}{\cite{huetal}} & \multicolumn{1}{l}{\cite{ldpfl}} & \multicolumn{1}{l}{This Work} \\
\midrule
Central Privacy & \checkmark & \checkmark & \checkmark & \checkmark \\
\hline
Local Privacy   & $\times$ & \checkmark & \checkmark & \checkmark \\
\hline
Personalization & $\times$ & \checkmark & $\times$ & \checkmark \\
\hline
Mild Assumptions on Training & \checkmark & $\times$ & \checkmark & \checkmark \\
\hline
\end{tabular}
\caption{Qualitative comparison with the most relevant prior research on the topic. More details provided in Section \ref{section:Related}.}
\label{tab:comp}
\end{table}

\section{An algorithm for private and personalized federated learning}
\label{section:Algorithm}

The following section introduces our proposed algorithm for federated learning with local guarantees to provide group privacy (Algorithm~\ref{alg:pifca}). Locality refers to the sanitization of the information released by the client to the server, whereas group privacy refers to indistinguishability with respect to a neighborhood of
clients defined with respect to a certain distance metric. Algorithm~\ref{alg:pifca} is motivated from the Iterative Federated Clustering Algorithm (IFCA)~\cite{ghosh} and builds on top of it to provide formal privacy guarantees. The main differences lie in the introduction of the \texttt{SanitizeUpdate} function described in Algorithm~\ref{alg:sanitize} and $k$-means for server-side clustering of the updated models.

The optimization strategy adopted here for personalization of the federated models is discussed in the works of \cite{ghosh} and \cite{mansour} which converge to proposing similar algorithms independently. In summary, the intuition is to initialize a set of hypotheses for the parameter vectors, one for each potential cluster. In the $t^{\text{th}}$ iteration, a subset of users receives the hypotheses, following which, each participating user determines which one of the them to optimize by evaluating which parameter vector yields the lowest cost over the local dataset. The assumption is that users with similar data distributions will adopt the same hypothesis. The updated models are then privatized before being returned to the server for averaging. The server is now tasked with deciding which models belong to the same cluster, in order to aggregate the corresponding parameter vectors. To do so, it performs $k$-means clustering starting from a specific choice of
centroids, providing fast convergence. Estimating the clusters is effective under the assumption that the sanitized update to the model parameters $\hat{\delta}^{(t)}_c$ is relatively smaller than the difference between hypotheses at time $t$. With the notation described in Equations \eqref{erm:1} through \eqref{erm:5} and adopted in Algorithm \ref{alg:pifca}, it means that $\forall\,j,i\,\in\,\left[ k \right],\,j = \bar{j},\,j\neq i,\;\forall\,c\,\in C^{(t)}$:
\begin{equation}
  \label{condition1}
  \hat{\delta}^{(t)}_c \defeq \norm{\hat{\theta}_{\bar{j},c}^{(t)} - \theta_{j}^{(t)}}_2 \ll
  \norm{\theta_{i}^{(t)} - \theta_{j}^{(t)}}_2
\end{equation}

It is possible to see experimentally that these assumptions are mild and typically verified with machine learning models with a small number of parameters and a careful tuning of the Laplacian
noise, although the optimal hypotheses depend of course on the (unknown) data
distributions.

To introduce privacy guarantees in Algorithm~\ref{alg:pifca}, we deviate from the standard implementation of IFCA~\cite{ghosh,mansour} in the following ways:
\begin{enumerate}
  \item We expect all the information leaving the users to be obfuscated locally before reaching the server.
  \item Information about the number of samples a user trained the model on is not disclosed at all.
  \item Users do not communicate the cluster membership to the server. This would be yet another information to sanitize, and we opt instead for letting the server evaluate membership based on the already privatized parameter vectors.
  \item It follows that users cannot communicate $\hat{\delta}^{(t)}_c$ but the full sanitized and updated parameter vector $\hat{\theta}_{\bar{j},c}^{(t)}$. In other words, Algorithm~\ref{alg:pifca} cannot rely on gradient averaging~\cite{ghosh} and resorts to model averaging.
\end{enumerate}

\begin{algorithm}
  \caption{An algorithm for personalized federated learning with formal privacy guarantees in local neighborhoods.}
  \label{alg:pifca}
  \begin{algorithmic}[1]
  \Require number of clusters $k$; initial hypotheses $\theta_j^{(0)},
  j \in \left[ k \right]$; number of rounds $T$; number of users per
  round $U$; number of local epochs $E$; local step size $s$;
  user batch size $B_s$; noise multiplier $\nu$; local dataset $Z_c$ held by user $c$.
  \For{$t = \left\{ 0, 1, \dots, T-1 \right\}$} \Comment{Server-side loop}
  \State $C^{(t)} \gets$ SampleUserSubset($U$)
  \State BroadcastParameterVectors($C^{(t)}$; $\theta_j^{(t)}, j \in \left[k
  \right]$)
    \For{$c \in C^{(t)}$} in parallel \Comment{Client-side loop}
    \State $\bar{j} = \argmin_{j \in \left[ k \right]}F_c(\theta_j^{(t)};
    Z_c)$
    \State $\theta_{\bar{j}, c}^{(t)} \gets$ LocalUpdate($\theta_{\bar{j}}^{(t)};
      s; E; Z_c$)
      \State $\hat{\theta}_{\bar{j}, c}^{(t)} \gets$
      SanitizeUpdate($\theta_{\bar{j}, c}^{(t)}$; $\nu$)
    \EndFor
    \State $\left\{S_1, \dots, S_k\right\} = \text{k-means}$($\hat{\theta}_{\bar{j},c}^{(t)}$,
      $c \in C^{(t)}$; $\theta_j^{(t)}, j \in \left[k \right]$)
    \State $\theta_j^{(t+1)} \gets \frac{1}{|S_j|}\sum_{c \in S_j}
    \hat{\theta}_{\bar{j}, c}^{(t)}, \quad \forall j \in \left[ k \right]$
  \EndFor
\end{algorithmic}
\end{algorithm}

\begin{algorithm}
  \caption{SanitizeUpdate obfuscates a vector $\theta \in \mathbb{R}^n$,
  with a Laplacian noise tuned on the radius of a certain neighborhood and centered
  in $0$.}
  \label{alg:sanitize}
  \begin{algorithmic}[1]
    \Function{SanitizeUpdate}{$\theta_{\bar{j}}^{(t)}; \theta_{\bar{j}, c}^{(t)}; \nu$}
    \State $\delta_c^{(t)} = \theta_{\bar{j}, c}^{(t)} - \theta_{\bar{j}}^{(t)} $
    \State $\varepsilon = \frac{n}{\nu \Vert\delta_c^{(t)}\Vert}$
    \State Sample $\rho \sim \mathcal{L}_{0, \varepsilon}(x)$
    \State $\hat{\theta}_{\bar{j}, c}^{(t)}  = \theta_{\bar{j}, c}^{(t)} + \rho$
    \State \Return $\hat{\theta}_{\bar{j}, c}^{(t)}$
    \EndFunction
  \end{algorithmic}
\end{algorithm}

\subsection{The Laplace mechanism under Euclidean distance in
$\mathbb{R}^n$}~\label{sec:laplace}
In Algorithm~\ref{alg:sanitize}, \texttt{SanitizeUpdate} requires a careful consideration as it is the main privacy preserving mechanism. 
All the server sees when a user communicates back is a parameter vector $\theta \in \mathbb{R}^n$. Without implementing the privacy mechanism, the true value from the user would be disclosed. Therefore, the following part of the section presents the motivation for and derivation of a particular flavor of the Laplace mechanism, and the heuristic used in \texttt{SanitizeUpdate} to define the neighborhood of a client.

\subsubsection{Motivation}
From the literature on geo-indistinguishability~\cite{geo}, we extend the Laplace mechanism with Euclidean distance for any metric space $\mathbb{R}^n$ as described in Lemma~\ref{lemma_laplace}. Note that there is no univocal definition of the multivariate Laplace distribution, and many different results can be
considered generalizations of the univariate case. We resort to the Laplace mechanism under Euclidean distance because of the two following reasons:
\begin{enumerate}
    \item[i)] Clustering is performed on $\theta$ with the $k$-means algorithm under Euclidean distance. Since we define clusters or groups of users based on how close their model parameters are under $L_2$ norm, we are looking for a $d$-privacy mechanism that obfuscates the reported values within a certain group and allows the server to differentiate among users belonging to different clusters.
    \item[ii)] Consider an input-output relation of the kind $y = f(x, \theta)$ with $f$ differentiable with respect to $\theta$. Its parameter vector $\theta$ is to be estimated with Algorithm~\ref{alg:pifca}, 
    such that it minimizes the Root Mean Square Error (RMSE) cost function
    \begin{equation} \label{RMSE}
      F_c = \sqrt{\frac{\sum\limits_{i=1}^{|Z_c|} (y_i - f(x_i,
        \theta))^2}{|Z_c|}} = \frac{\norm{Y - f(X, \theta)}_2}{\sqrt{|Z_c|}}
    \end{equation}
     and $X = \left[ x_1, \dots, x_{|Z_c|}
    \right]^T$, $Y = \left[ y_1, \dots, y_{|Z_c|} \right]^T$, with $|Z_c|$ being the number 
    of data points held by client $c$. If a client releases to the server
    its parameters $\theta_c$ sanitized by addition of random vector
    $\Delta$, we can evaluate how the cost function would change with respect
    to the non-sanitized communication. Dropping the multiplicative constant
    we find:
    \begin{equation}
      \begin{split}
      \norm{Y - f(X, \theta_c)}_2 - \norm{Y - f(X, \theta_c + \Delta)}_2 &
      \leq \\
      \norm{Y - f(X, \theta_c) - Y + f(X, \theta_c + \Delta) }_2 & = \\
      \norm{f(X, \theta_c + \Delta) - f(X, \theta_c)}_2 & \approx \\
      \norm{f(X, \theta_c) + \nabla f(X, \theta_c)^T\Delta -f(x, \theta_c)}_2 & = \\
      \norm{\nabla f(X, \theta_c)^T\Delta}_2 & \leq \\
      \norm{\nabla f(X,\theta_c)}_2\norm{\Delta}_2
      \end{split}
    \end{equation}
    Hence, we notice how we can bound such value proportionally to the
    Euclidean norm of the random noise. Notably, it does not depend on the
    direction of $\Delta$. Thus, we require that points with the same bound
    on the increase of the cost function (which are all points distant
    $\norm{\Delta}_2$ from $\theta_c$) will be sampled with the same
    probability.
\end{enumerate}

\subsubsection{Derivation}
\begin{lemma}\label{lemma_laplace} Let $\mathcal{L}_{\varepsilon}:
  \mathbb{R}^n \rightarrow
  \mathbb{R}^n$ be the Laplace mechanism of the form
  $\mathcal{L}_{x_0,\varepsilon}(x) = \mathbb{P}\left[\mathcal{L}_{\varepsilon}(x_0)=x\right] = Ke^{-\varepsilon d(x,x_0)}$ with
$d(x,x_0)=\norm{x-x_0}_2$. The mechanism is $\varepsilon$ $d$-private and
\begin{equation}
  K = \frac{\varepsilon^n\Gamma(\frac{n}{2})}{2 \pi^{\frac{n}{2}}\Gamma(n)}
  \label{K}
\end{equation}
\end{lemma}

\begin{proof} If $\mathcal{L}_{x_0,\varepsilon}(x)=Ke^{-\varepsilon d(x,x_0)}$ is a
probability density function of a point in $\mathbb{R}^n$ then there is a
$K$ such that $\int_{\mathbb{R}^n}\mathcal{L}_{x_0}(x)dx = 1$. We note that it depends
only on the distance from $x_0$ and we can write $Ke^{-\varepsilon d(x,x_0)} =
Ke^{-\varepsilon r}$ where $r$ is the radius of the ball in
$\mathbb{R}^n$ centered in $x_0$. Without loss of generality, let us now take
$x_0=0$. The probability density of the event $ x
\in \mathbb{S}_n(r) = \{x : \norm{x}_2 = r\}$ is then $p(x \in
\mathbb{S}_n(r)) = Ke^{-\varepsilon r} \mathbb{S}_n(1)r^{n-1}$ where
$S_n(1)$ is the surface of the unitary ball in $\mathbb{R}^n$ and
$S_n(r) = S_n(1)r^{n-1}$ is the surface of a generic ball of radius $r$.
Given that
\begin{equation}
  S_n(1) = \frac{2 \pi^{n/2}}{\Gamma(\frac{n}{2})}
\end{equation}
solving
\begin{equation}
\begin{split}
  \int_0^{+\infty} \mathbb{P}\left[x \in \mathbb{S}_n(r)\right] dr
  &= \int_0^{+\infty} Ke^{-\varepsilon r} S_n(1)r^{n-1}dr = \\
  &= K\frac{2 \pi^{n/2}\Gamma(n)}{\varepsilon^n\Gamma(\frac{n}{2})} =
  1
\end{split}
\end{equation}
results in
\begin{equation}
  K = \frac{\varepsilon^n\Gamma(\frac{n}{2})}{2 \pi^{\frac{n}{2}}\Gamma(n)}
\end{equation}
where $\Gamma(\cdot)$ denotes the gamma function.
By plugging $\mathcal{L}_{x_0,\varepsilon}(x)=Ke^{-\varepsilon d(x,x_0)}$ in Equation
\ref{eq:dprivacy}:
\begin{equation}
  Ke^{-\varepsilon d(x,x_1)} \leq e^{\varepsilon d(x_1,x_2)}Ke^{-\varepsilon
  d(x,x_2)}
\end{equation}
\begin{equation}
  e^{\varepsilon (\norm{x-x_2}_2- \norm{x-x_1}_2)} \leq
  e^{\varepsilon\norm{x_1-x_2}} =
  e^{\varepsilon d(x_1,x_2)}
\end{equation}
\end{proof}

One of the biggest advantages of $d$-privacy is that the level of privacy can be derived for a repeated number of independent queries due to the the fact that it satisfies the compositionality theorem~\cite{DworkDP_Compositionality}, which is one of the key properties for the applicability of DP and its variants for formalizing the privacy guarantee for a composition of independent queries.

\begin{restatable}{theorem}{compositionality}\label{th:compositionality}[Compositionality Theorem for $d$-privacy]
Let $\mathcal{K}_i$ be $(\varepsilon_i)$-$d$-private mechanism for $i\in \{1,2\}$. Then their independent composition is $(\varepsilon_1+\varepsilon_2)$-$d$-private, i.e., for every $S_1,\,S_2\subseteq \mathcal{Y}$ and all $x_1,\,x'_1,\,x_2,\,x'_2\,\in\,\mathcal{X}$, we have:
\begin{align}\label{eq:composition}
    \mathbb{P}_{\mathcal{K}_1,\mathcal{K}_2}\left[(y_1,y_2)\,\in\, S_1\times S_2\vert (x_1,x_2)\right]\nonumber\\
    \leq e^{\varepsilon_1\,d(x_1,x'_1)+\varepsilon_2\,d(x_2,x'_2)}\mathbb{P}_{\mathcal{K}_1,\mathcal{K}_2}\left[(y_1,y_2)\,\in\, S_1\times S_2\vert (x'_1,x'_2)\right]
\end{align}
\end{restatable}

\begin{proof}
Let us simplify the notation and denote: $$P_i=\mathbb{P}_{\mathcal{K}_i}\left[y_i\in S_i\vert x_i\right]$$ $$P'_i=\mathbb{P}_{\mathcal{K}_i}\left[y_i\in S_i\vert x'_i\right]$$
for $i\,\in\,\{1,2\}$.
As mechanisms $\mathcal{K}_1$ and $\mathcal{K}_2$ are applied independently, we have:
\begin{align}
    \mathbb{P}_{\mathcal{K}_1,\mathcal{K}_2}\left[(y_1,y_2)\,\in\, S_1\times S_2\vert (x_1,x_2)\right]=P_1.P_2 \nonumber\\
    \mathbb{P}_{\mathcal{K}_1,\mathcal{K}_2}\left[(y_1,y_2)\,\in\, S_1\times S_2\vert (x'_1,x'_2)\right]=P'_1.P'_2 \nonumber
\end{align}
Therefore, we obtain:
\begin{align}
     \mathbb{P}_{\mathcal{K}_1,\mathcal{K}_2}\left[(y_1,y_2)\,\in\, S_1\times S_2\vert (x_1,x_2)\right]=P_1.P_2\nonumber\\
     \leq\left(e^{\varepsilon_1\,d(x_1,x'_1)}P'_1\right)\left(e^{\varepsilon_2\,d(x_2,x'_2)}P'_2\right)\nonumber\\
     \leq e^{\varepsilon_1\,d(x_1,x'_1)+\varepsilon_2\,d(x_2,x'_2)}\mathbb{P}_{\mathcal{K}_1,\mathcal{K}_2}\left[(y_1,y_2)\,\in\,
     S_1\times S_2\vert (x'_1,x'_2)\right]\nonumber
\end{align}
\end{proof}

\subsubsection{A heuristic for defining the neighborhood of a client} \label{heuristic}
In the $t^{\text{th}}$ iteration, when a user $c$ calls the \texttt{SanitizeUpdate} routine in Algorithm~\ref{alg:sanitize}, it has already received a set of hypotheses, optimized $\theta_{\bar{j}}^{(t)}$ (the one that fits best its data distribution),
and got $\theta_{\bar{j}, c}^{(t)}$. It is reasonable to assume that clients whose 
datasets are sampled from the same underlying data distribution $\mathcal{D}_{\bar{j}}$
(as described in Section \ref{flp}) will perform an update similar to $\delta_c^{(t)}$.
\begin{definition}\label{def:neighbourhood}
For any model parametrized by $\theta\,\in\,\mathbb{R}^n$, we define its $r$-\emph{neighborhood} as the set of points in the parameter space which are at a $L_2$ distance of at most $r$ from $\theta$, i.e., $\{\phi \in \mathbb{R}^n \colon \left\Vert \theta,\phi\right\Vert_2\leq r\}$
\end{definition}

\begin{definition}\label{def:group}
Clients whose models are parametrized by $\theta \in \mathbb{R}^n$ in the same $r$-neighborhood are said to be in the same \emph{group}, or \emph{cluster}.
\end{definition}

Therefore, we require that points which are within the $\delta_c^{(t)}$-neighborhood of $\hat{\theta}_{\bar{j}, c}^{(t)}$ to be indistinguishable. To provide this guarantee, we tune the Laplace mechanism such that the points within the neighborhood are $\varepsilon \Vert \delta_c^{(t)}\Vert_2$ differentially private.
With the choice of $\varepsilon = n/(\nu \delta_c^{(t)})$, one finds that
$\varepsilon \Vert \delta_c^{(t)}\Vert_2 = n/\nu$, and we call $\nu$ the \emph{noise multiplier}. It is straightforward to observe that the larger the value of $\nu$ gets, the stronger is the privacy guarantee. Note that in order to derive this result, we exploited the fact that the norm of the noise vector sampled from  
Laplace distribution is distributed according to Equation \eqref{d_r} and its expected value is $\mathbb{E}\left[ \gamma_{\varepsilon, n}(r)\right] = n/\varepsilon$.

\subsection{Sampling from the Laplace mechanism} \label{section:Sampling}
Exploiting the radial symmetry of the Laplace distribution, we note that, in
order to sample a point $x_s \sim \mathcal{L}_0(x)$ in  $\mathbb{R}^n$, it is possible to first sample the set of points distant $d(x, 0)=r$ from $x_0=0$ and then sample
uniformly from the resulting hypersphere. Accordingly, the p.d.f. of the event $ x \in \mathbb{S}_n(r) = \{x : \norm{x}_2
= r\}$ is then $\mathbb{P}\left[x \in \mathbb{S}_n(r)\right] = Ke^{-\varepsilon r} \mathbb{S}_n(1)r^{n-1}$, where $K$ is as in Lemma~\ref{lemma_laplace} and $\mathbb{S}_n(r)$ is the surface of the sphere with radius $r$ in $\mathbb{R}^n$. Hence, we can write
\begin{equation}
  \gamma_{\varepsilon,n}(r) = \frac{\varepsilon^n e^{-\varepsilon r}
  r^{n-1}}{\Gamma(n)}
  \label{d_r}
\end{equation}
which is the gamma distribution with shape $n$ and scale $1/\varepsilon$. Drawing from $\gamma_{\varepsilon, n}(r)$ is implemented in multiple routines
in common programming languages. Equation \eqref{d_r} represents the p.d.f. of sampling the hypersphere of radius $\norm{x_s} = r \sim
\gamma_{\varepsilon, n}(r)$. To sample a point uniformly from the corresponding hypersphere one can sample ${}_{1}x_n \in \mathbb{S}_n(1)$, a point from the hypersphere of radius $1$, and have that $x_s = {}_{1}x_n\norm{x_s}$, where ${}_{1}x_n = \frac{x_n}{\norm{x_n}}$. This can be
done operationally by sampling $x_n$ from the $n$-dimensional vector whose components are sampled from a Gaussian
distribution centered at $0$ and with a variance $\sigma^2$, i.e., $x_n \sim \mathcal{G}_n(0, \sigma^2)$ and letting ${}_{1}x_n = \frac{x_n}{\norm{x_n}}$.

\subsection{Component-wise variance}
In order to better characterize the distribution in Lemma \ref{lemma_laplace}, we now proceed to show how to derive the variance of each single component $x_i$ of 
$x = \left[x_1, \dots, x_n\right]^T$.
\begin{lemma}
  Let $x \sim \mathcal{L}_{0, \varepsilon}$, $x \in \mathbb{R}^n$ as in Lemma
  \ref{lemma_laplace}
  and $r \sim \gamma_{\varepsilon, n}$ as in Equation \eqref{d_r}, then we have that the variance of
  the $i$-th component of $x$ is $\sigma_{x_i}^2 =
  \frac{n+1}{\varepsilon^2}$.
\end{lemma}
\begin{proof}
  With $r \sim \gamma_{\varepsilon, n}$ we have that, by construction,
  \begin{equation}
  \mathbb{E}\left[ r^2 \right] = \mathbb{E}\left[ \sum_{i=1}^n x_i^2\right]
  = n\mathbb{E}\left[ x_i^2 \right] = n\sigma_{x_i}^2
  \end{equation}
  With the last equality holding since $\mathcal{L}_{0,\varepsilon}$ is
  isotropic and centered in zero.
  Recalling that
  \begin{equation}
    \mathbb{E}\left[ r^2 \right] =
    \left. \frac{d^2}{dt^2} M_r(t) \right|_{t=0}
  \end{equation}
  with $M_r(t)$ the moment generating function of the gamma distribution
  $\gamma_{\varepsilon,n}$,
  \begin{equation}
    \begin{split}
      & \left.\frac{d^2}{dt^2} \left( \left( 1- \frac{t}{\varepsilon}\right)^{-n}
        \right)\right|_{t=0} = \nonumber\\
     = &\left.\frac{n(n+1)}{\varepsilon^2}\left( 1-
  \frac{t}{\varepsilon}\right)^{-(n+2)} \right|_{t=0} = \\
   = & \frac{n(n+1)}{\varepsilon^2}
  \end{split}
  \end{equation}
  which leads to
  \begin{equation}
  \sigma_{x_i}^2 = \frac{n+1}{\varepsilon^2}
  \end{equation}
\end{proof}

\subsection{Limitations of the Laplace mechanism in very high dimensional spaces}\label{sec:limitations}

As already described in Section \ref{sec:dprivacy} and \ref{sec:laplace}, $d$-privacy provides differential
privacy guarantees to a point $x_0 \in \mathcal{X}$, with privacy parameter at most $\varepsilon r$, with
respect to any point $x$, such that $d(x, x_0) \leq r$. 
These local differential/$d$-privacy guarantees for federated learning models are a desirable
feature which would make any information disclosure from the client to the server 
indistinguishable up to a certain multiplicative factor. Local DP mechanisms ensure also
central DP, and thus would provide its guarantees as well. However, LDP is notoriously
hard to achieve while maintaining utility of the queries. In \cite{ldp-bounds}
are evaluated the lower bounds of the error on the estimate of a counting query under both local and
central DP with the Laplace mechanism. They are found to be $O(1/ \varepsilon)$
and $\Omega(\sqrt{N}/ \varepsilon)$ respectively, which for the latter depend 
on the number of participating individuals $N$.
In the context of federated learning though, where individual information is aggregated 
e.g. by average, the Central Limit Theorem would yield a reduction of the standard 
deviation of the aggregate error by $\sqrt{N}$ in the local model. Instead, we want to highlight what
we consider to be the hardest obstacle in providing LDP guarantees in federated learning.

Assume that we  want to sanitize information locally with the Laplace mechanism defined 
in Lemma \ref{lemma_laplace}. With the results found in Section 
\ref{section:Sampling} we see that each point $x \in \mathbb{R}^n$ would be sanitized by addition of 
a vector $\rho$ whose norm is distributed as $\norm{\rho}_2 \sim \gamma_{\varepsilon,n}(r)$.
Its mean is found to be $\mathbb{E}\left[\gamma_{\varepsilon,n}(r)\right] =n/\varepsilon$, and we 
highlight the linear dependency on $n$. In large machine learning models where the number of
parameters easily reaches a few millions, this would completely destroy utility, as maintaining LDP with small $\varepsilon$ values would require noise levels that dwarf the true values of the parameters.
Indeed, in Section \ref{nn} we conduct  experiments on model architecture leading to $\theta \in \mathbb{R}^{1206590}$, and we can see that maintaining low levels of 
the LDP parameters  would destroy the model's accuracy. Conversely, maintaining high utility would yield huge values of LDP parameters, rendering formal LDP guarantee practically meaningless. However, in the case of machine learning, the typical white-box attack is the Deep Leakage from Gradients (DLG)~\cite{deep-leakage}. In our experiments, we have empirically verified that we can achieve a strong defense against this kind of attacker while maintaining a good level of accuracy.

\section{Experiments}
\label{section:Experiments}

\subsection{Synthetic data}\label{sec:synthetic}
The first experiment tests Algorithm~\ref{alg:pifca} on synthetic data generated from a
linear mapping with a set of 
predetermined optimal parameters. In particular, we generate data according to $k=2$ different
distributions
\begin{equation}
    y = x^T\theta_1^* + u; \quad u \sim \text{Uniform}\left[0, 1\right)
\end{equation}
\begin{equation}
    y = x^T\theta_2^* + u; \quad u \sim \text{Uniform}\left[0, 1\right)
\end{equation}
with $\theta_1^* = \left[+5, +6\right]^T, \theta_2^* = \left[+4, -4.5\right]^T$. A total of $100$ users
holds $10$ samples each, drawn from either one of the distributions. They participate in a 
training of two initial hypotheses which are sampled from a Gaussian distribution centered in $0$
and unit variance at iteration $t=0$. A total of $U=7$ users are asked to participate in the
optimization at each round and train locally the hypothesis that fits better their dataset for $E=1$
epochs each time. The noise multiplier is set to $\nu=5$.
Local step size $s = 0.1 $ and a batch size  $B_s = 10$ complete the required
inputs to the algorithm. To verify the training process, another set of users with the same
characteristics is held out form training to perform validation and stop the federated optimization
once the is no improvement in the loss function in Equation \eqref{RMSE} for $6$ consecutive rounds.
Results of the training process are shown in Figures \ref{fig1:31}, \ref{fig1:32}, \ref{fig1:33}. 
Note that the real clients parameters would not be visible to the server but are drawn on the plots
for
clarity. Although at first the updates seem to be distributed all over the domain, in just a few
rounds of training the process converges to values very close to the two optimal parameters.
With the heuristic presented in Section \ref{heuristic} it is easy to find that whenever a user
participates in an optimization round it incurs in a privacy leakage of at most $n/\nu = 2/5=0.4$,
in a differential private sense, with respect to points in its neighborhood. 
Using the result in Theorem \ref{th:compositionality} clients can compute the overall privacy leakage of the optimization
process, should they be required to participate multiple times. With the uniform sampling of the
clients (without replacement) that was used in this experiment, the maximum composed value of the
privacy leakage was $2.4$. For any user, whether to participate or not in a training round 
can be decided right before releasing the updated parameters, in case that would increase the
privacy leakage above a threshold value decided beforehand.

In a concise ablation study we assess how training progresses when two characteristic
features of Algorithm \ref{alg:pifca} are removed:
\begin{itemize}
    \item the privatization of the client parameters
    \item model personalization 
\end{itemize}
In Figure \ref{fig1:21}, \ref{fig1:22}, \ref{fig1:23} no sanitization is performed
on the updated parameters sent by the users and the optimization terminates with the 
clients very close to the optimal parameters. This is reflected in the validation loss
reaching the lowest value among the three cases. We highlight, though, how it is still in the same
order of magnitude as the sanitized case.

In Figure \ref{fig1:11}, \ref{fig1:12}, \ref{fig1:23} the clients are left to optimize the initial
hypotheses without personalization, and we find that the validation loss is considerably larger than both
the non-sanitized and sanitized case. This is evident also as the real client parameters transmitted
to the server converge to somewhere in between the optimal parameters.
Further, in Figure \ref{priv-leak-synth} is provided the increase in maximum value of privacy leakage clients incur into, per cluster.

\newcommand{\mywidth}{0.3}
\begin{figure*}
 \captionsetup[subfigure]{justification=centering}
 \centering
 \begin{subfigure}[b]{\mywidth\textwidth}
     \centering
     \includegraphics[width=\textwidth]{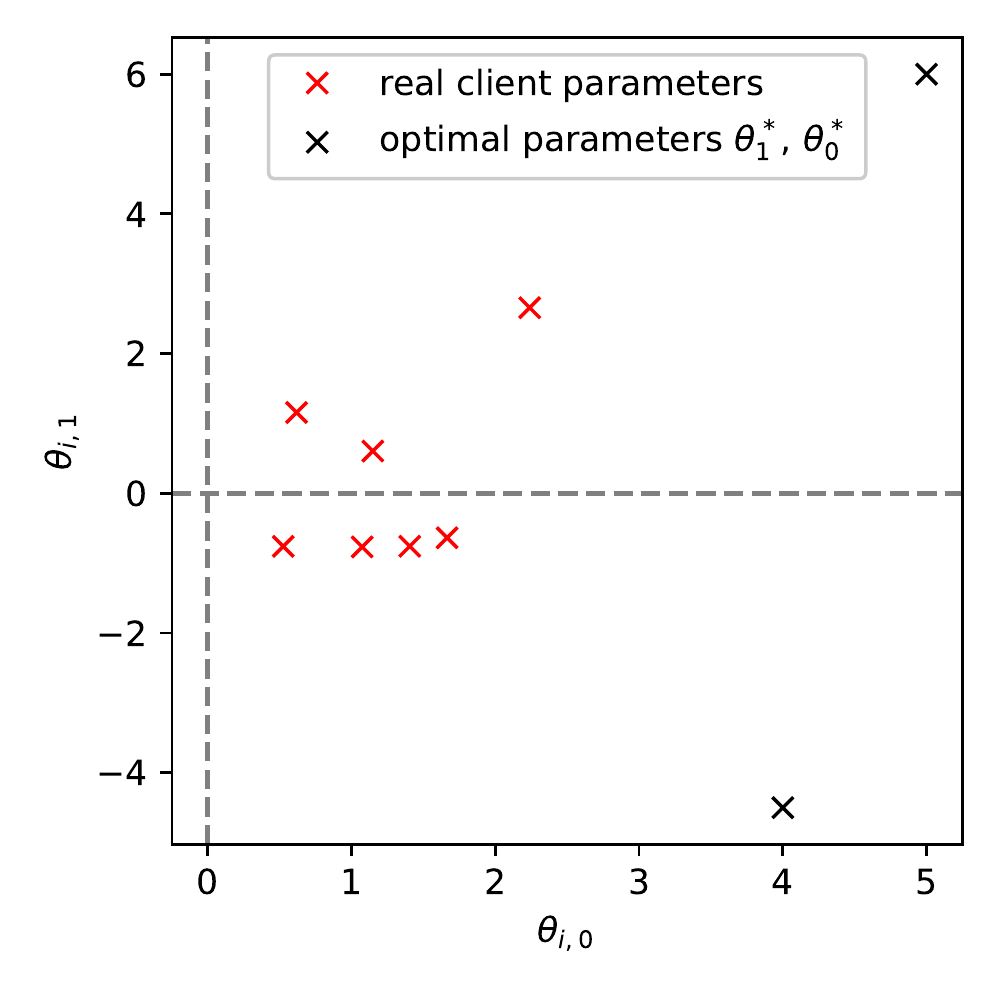}
     \caption{First round}
     \label{fig1:11}
 \end{subfigure}
 \hfill
 \begin{subfigure}[b]{\mywidth\textwidth}
     \centering
     \includegraphics[width=\textwidth]{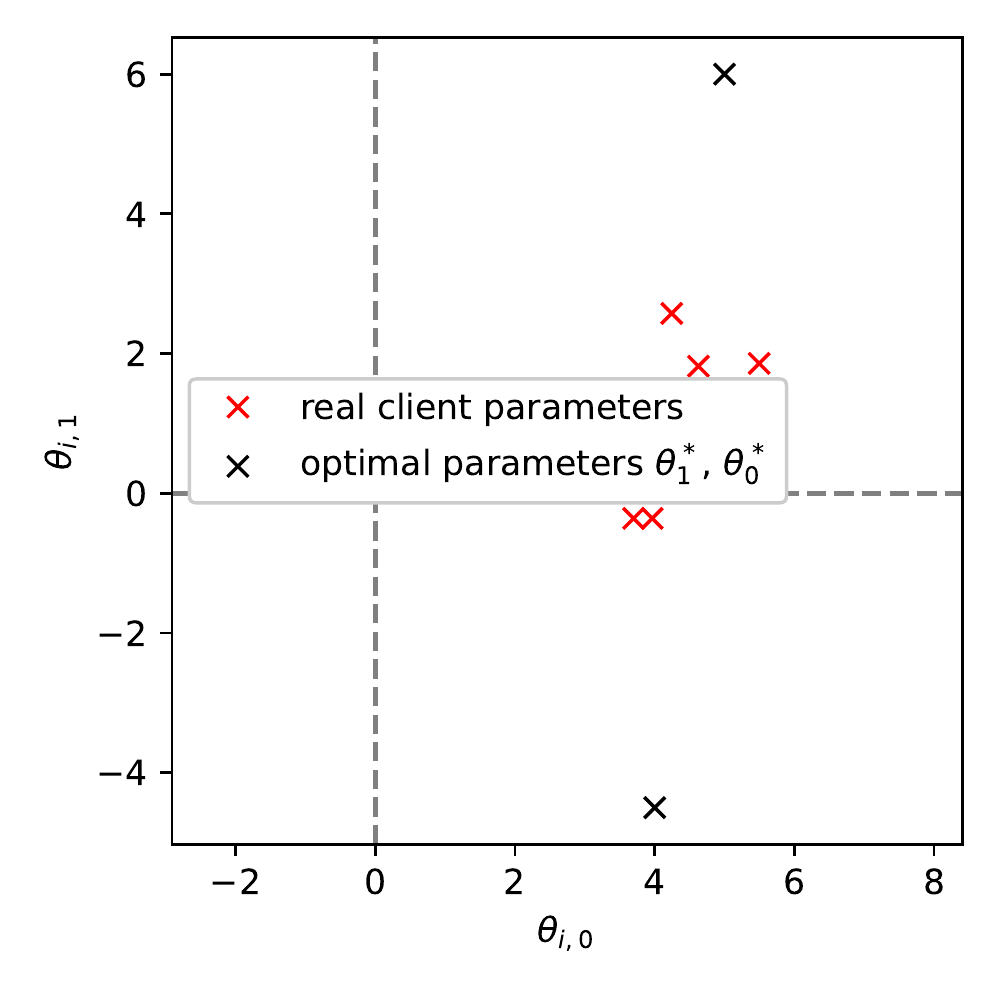}
     \caption{Best round}
     \label{fig1:12}
 \end{subfigure}
 \hfill
 \begin{subfigure}[b]{\mywidth\textwidth}
     \centering
     \includegraphics[width=\textwidth]{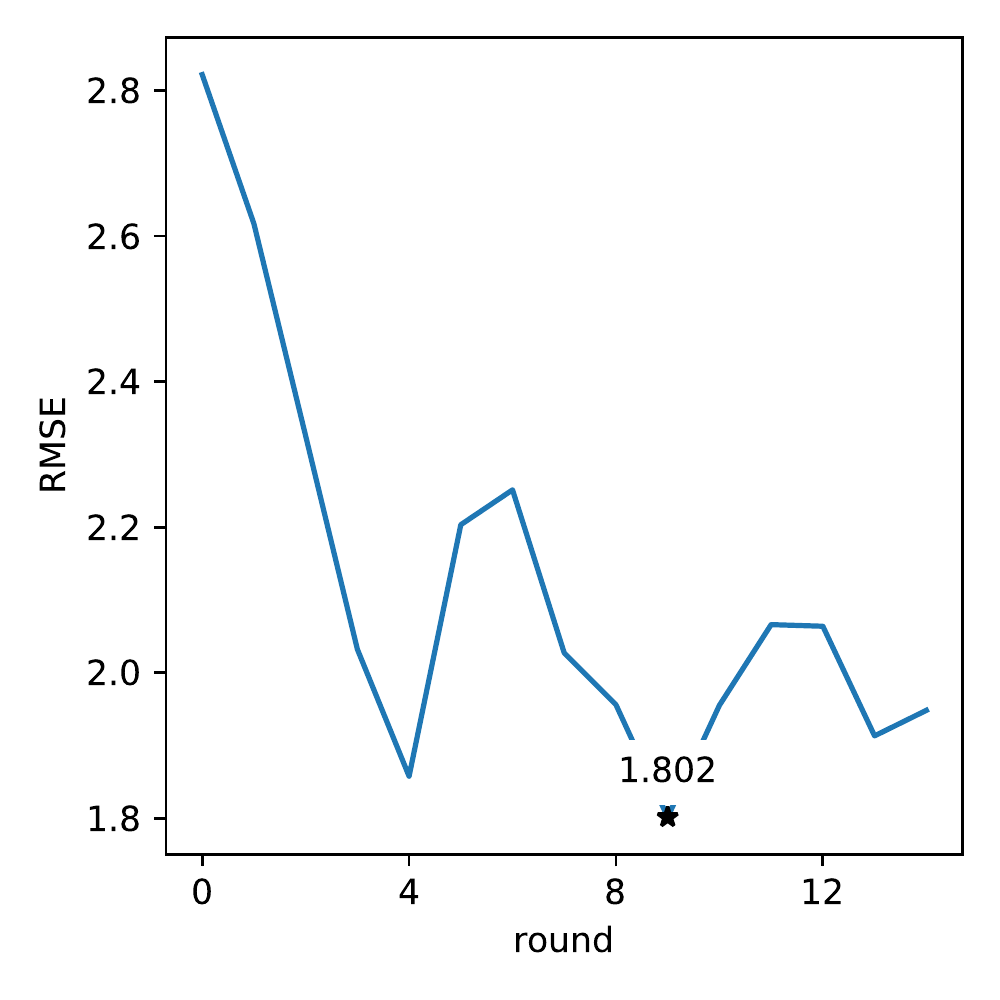}
     \caption{Validation loss}
     \label{fig1:13}
 \end{subfigure}

 \begin{subfigure}[b]{\mywidth\textwidth}
     \centering
     \includegraphics[width=\textwidth]{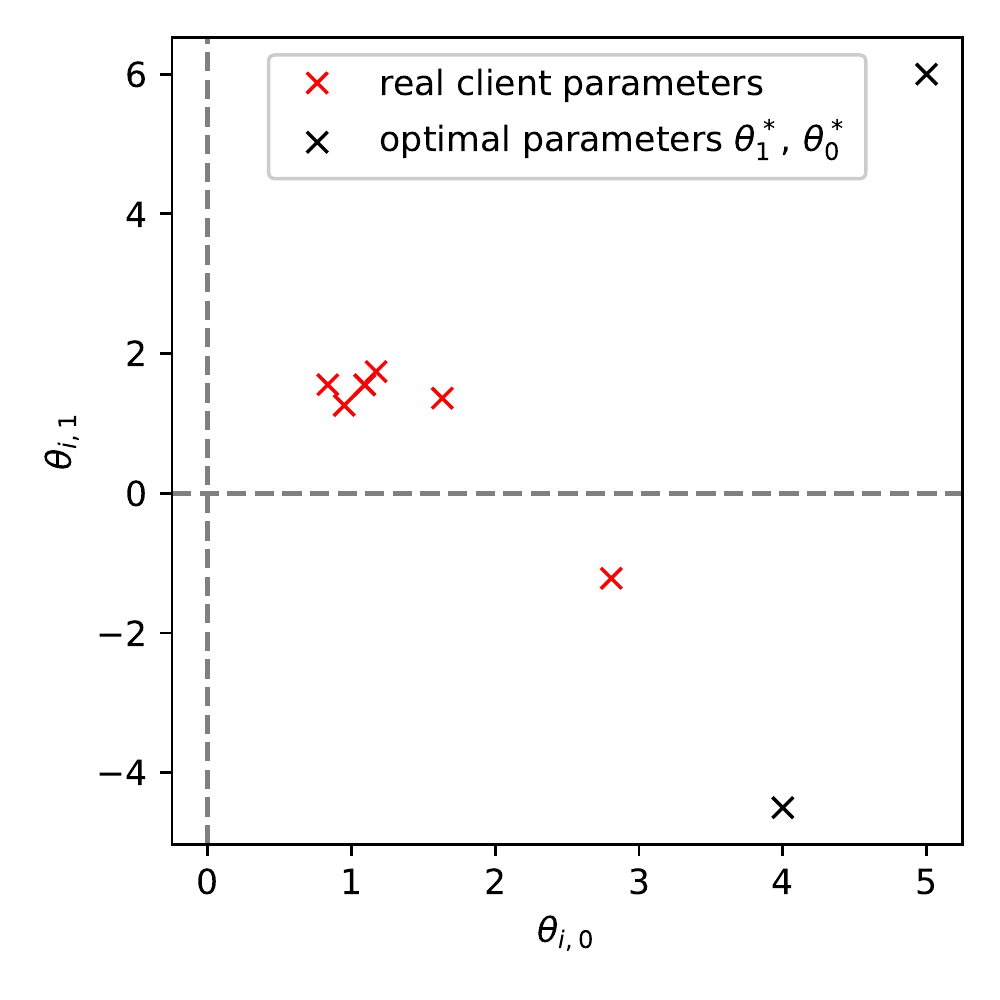}
     \caption{First round}
     \label{fig1:21}
 \end{subfigure}
 \hfill
 \begin{subfigure}[b]{\mywidth\textwidth}
     \centering
     \includegraphics[width=\textwidth]{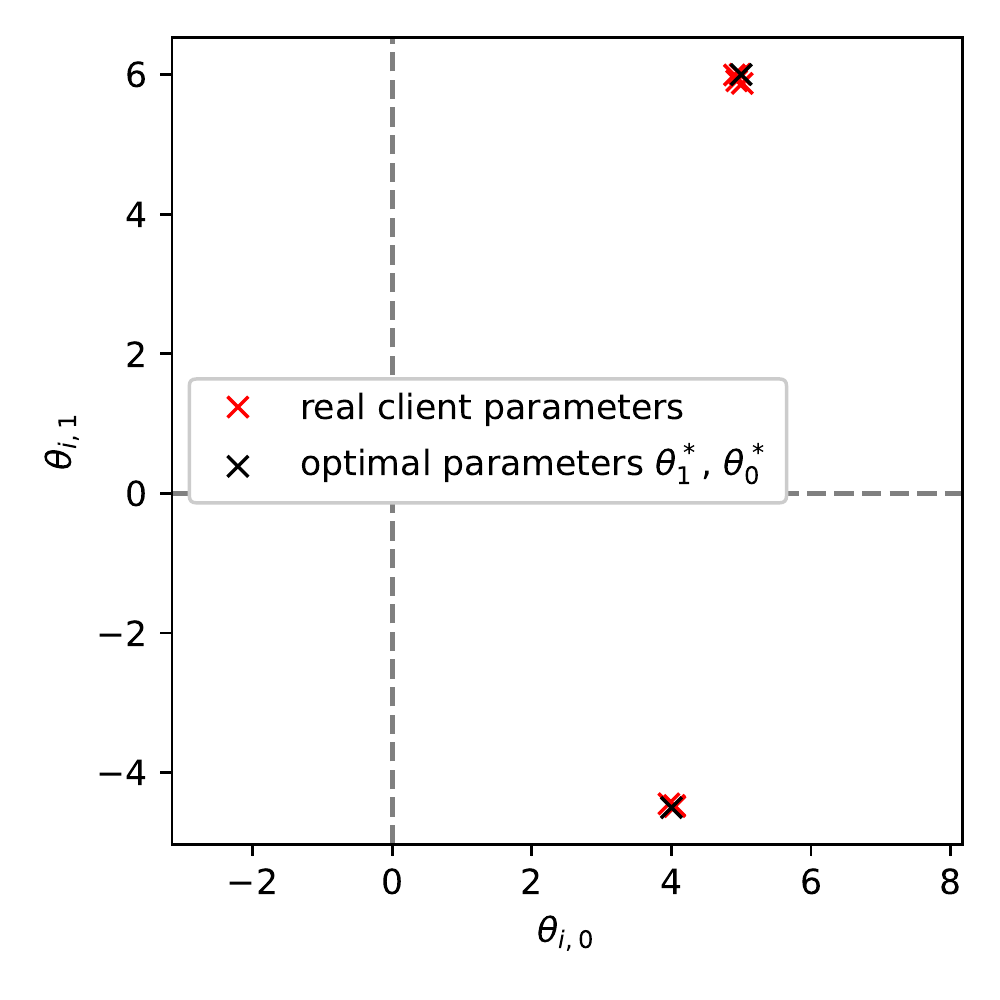}
     \caption{Best round}
     \label{fig1:22}
 \end{subfigure}
 \hfill
 \begin{subfigure}[b]{\mywidth\textwidth}
     \centering
     \includegraphics[width=\textwidth]{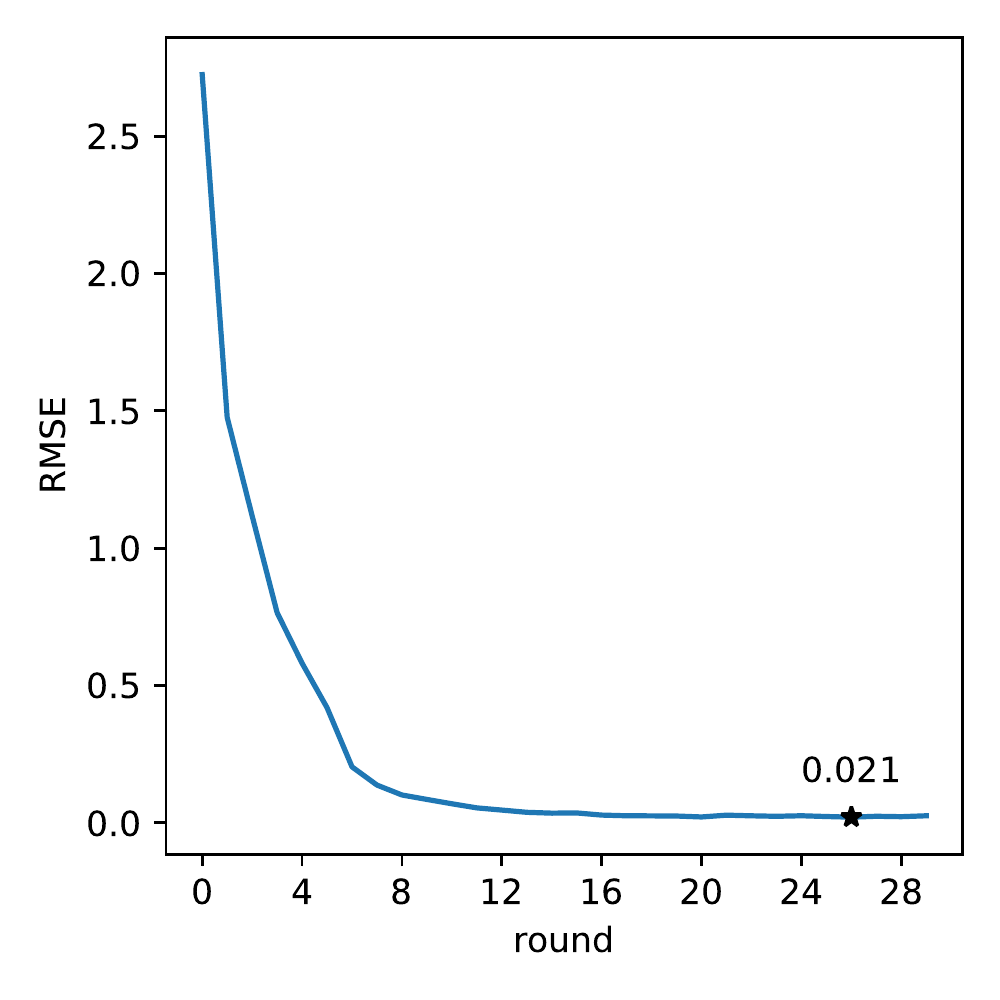}
     \caption{Validation loss}
     \label{fig1:23}
 \end{subfigure}

 \begin{subfigure}[b]{\mywidth\textwidth}
     \centering
     \includegraphics[width=\textwidth]{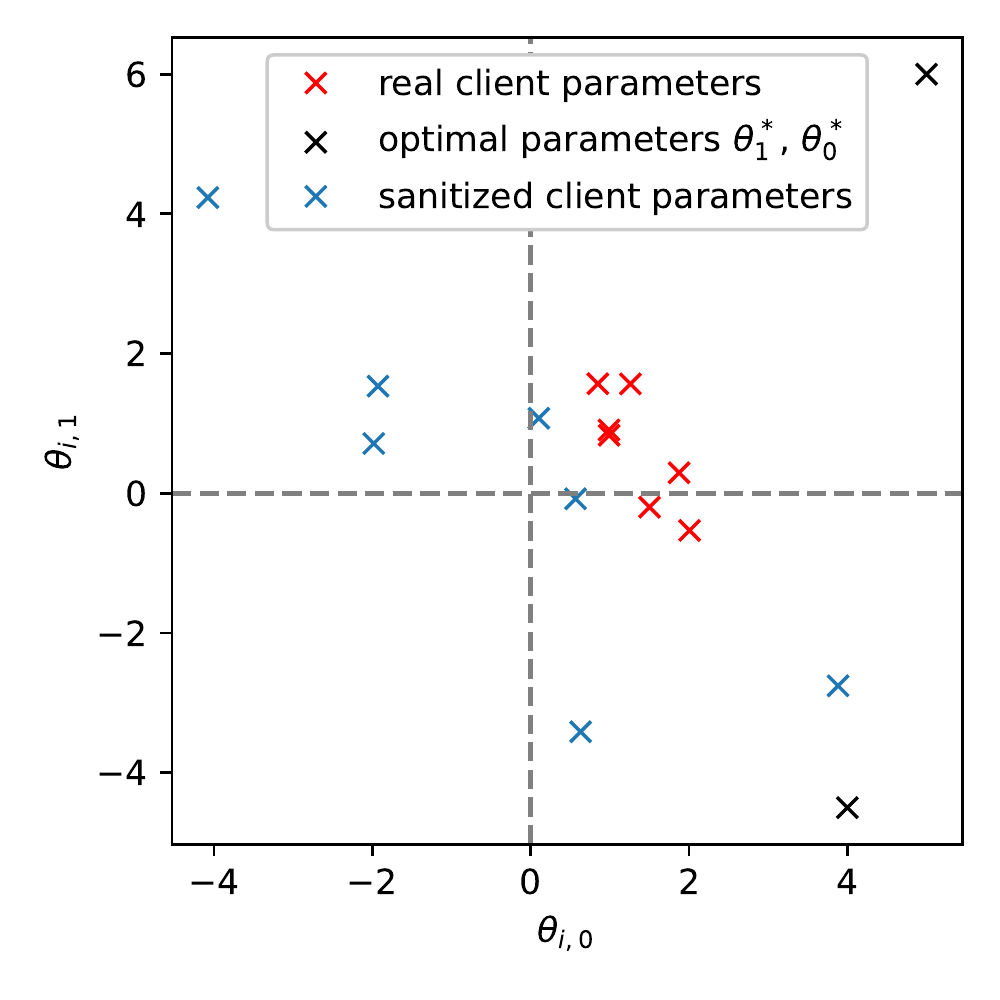}
     \caption{First round}
     \label{fig1:31}
 \end{subfigure}
 \hfill
 \begin{subfigure}[b]{\mywidth\textwidth}
     \centering
     \includegraphics[width=\textwidth]{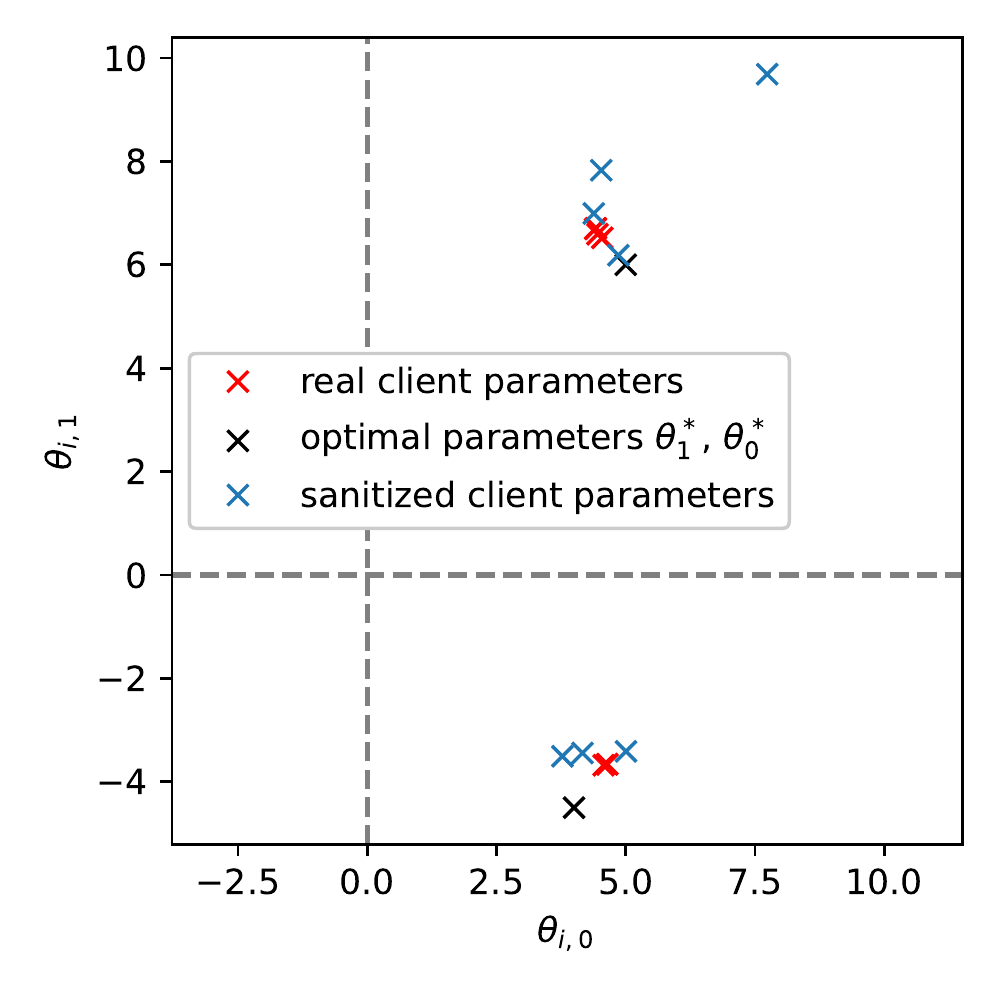}
     \caption{Best round}
     \label{fig1:32}
 \end{subfigure}
 \hfill
 \begin{subfigure}[b]{\mywidth\textwidth}
     \centering
     \includegraphics[width=\textwidth]{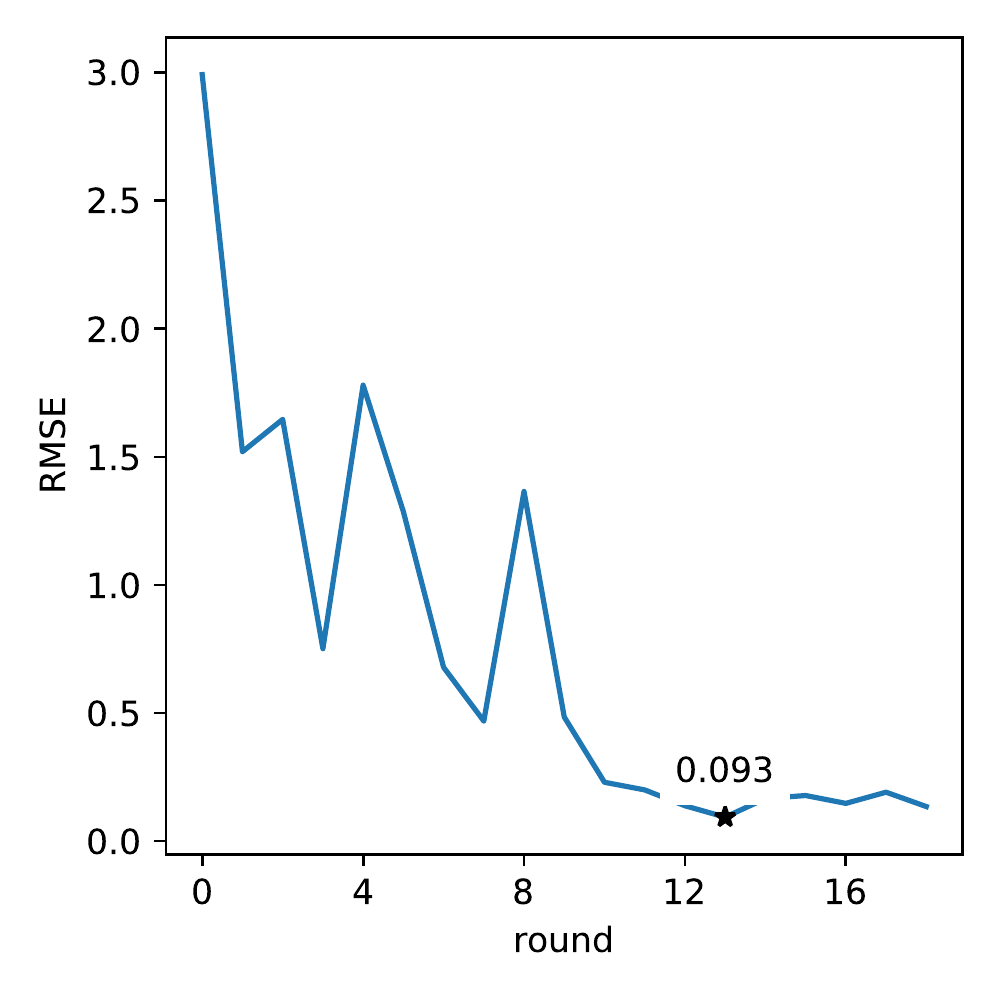}
     \caption{Validation loss}
     \label{fig1:33}
 \end{subfigure}
 \caption{Learning federated linear models with: (a, b, c) one initial hypothesis and
   non-sanitized communication, (d, e, f) two initial hypotheses and non-sanitized communication,
   (g, h, i) two initial hypotheses and sanitized communication. The first
 two figures of each row show the parameter vectors released by the clients
 to the server. The last figure of each row illustrates the trend of the
 validation loss on clients and data not involved in the optimization.}
    \label{fig1}
\end{figure*}

\begin{figure}[h]
\centering
\includegraphics[scale=0.55]{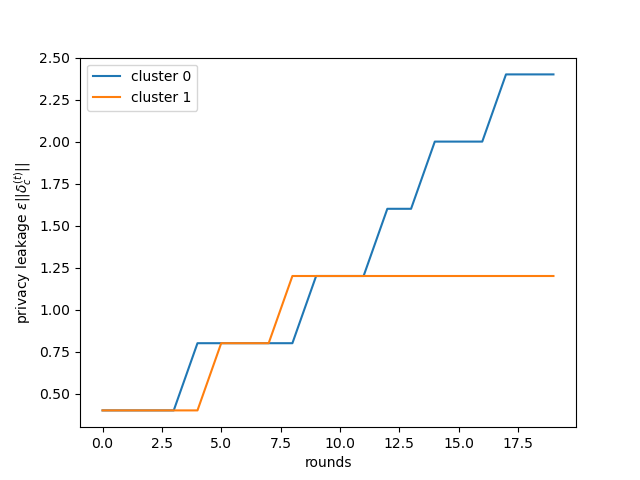}
\caption{For the experiment on synthetic data, this figure plots the max privacy leakage over clients of the same cluster for a round of training. Intervals with constant privacy leakage indicate that the clients with the largest privacy leakage were not sampled (by chance) to participate in those rounds.}
\label{priv-leak-synth}
\end{figure}

\subsection{Hospital charge data}\label{sec:hospital}
This experiment is performed on real world data, specifically, the Hospital Charge Dataset
published by the Centers for Medicare and Medicaid Services of the US Government. It contains
data about charges for the $100$ most common inpatient services and the $30$ most common 
outpatient services. It shows a great variety of charges applied by healthcare providers with
details mostly related to the type of service and the location of the provider.
Preprocessing of the dataset includes a number of procedures, the most important of which are
described here:
\begin{enumerate}
    \item[i)] Selection of the $4$ most widely treated conditions, which amount to simple
    pneumonia; kidney and urinary tract infections; hart failure and shock; esophagitis and
    digestive system disorders.
    \item[ii)] Transformation of ZIP codes into numerical coordinates in terms of longitude and
    latitude.
    \item[iii)] Setting as target the Average Total Payments, i.e. the cost of the service averaged
    among the times it was given by a certain provider.
    \item[iv)] As it is a standard procedure in the context of gradient-based optimization, 
    dependent and independent variables are brought to be in the range of the \emph{units}
    before being fed to the machine learning model. Note that this point takes the spot of the
    common feature normalization and standardization procedures,
    which we decided not to perform here
    to keep the setting as realistic as possible. In fact, both would require the knowledge of 
    the empirical distribution of all the data.
    Although it is available in simulation, it would not
    be available in a real scenario, as each user would only have access to their dataset.
\end{enumerate}
To simulate a federated learning process, healthcare providers are here considered
the set of clients willing to collaborate to train a machine learning model. Given the 
preprocessing described above, the dataset results in $2947$ clients, randomly split in train
and validation subsets with $70$ and $30$ per cent of the total clients each.
The goal is being able to predict the cost that a service would require given where it is 
performed in the country, and what kind of procedure it is. 
The model that was adopted in this context is a fully connected neural network (NN)
of two layers, with a total of $11$ parameters and Rectified Linear Unit (ReLU) activation
function. 
Inputs to the model are an increasing index which uniquely defines the healthcare service,
the longitude and latitude of the provider. Output of the model is the expected cost.
Tests have been performed to minimize the RMSE loss on the clients selected for training
($100$ per round) and at each round the performance of the model is checked against a held-out
set of validation clients, from where $200$ are sampled every time. If $30$ validation rounds
are passed without improvement in the cost function, the optimization process is terminated.
To assess the trade-off between privacy, personalization and accuracy, a different number of
initial hypotheses has been checked, as it is not known a-priori how many distributions
generated the data. For the same reason, accuracy has been checked at different values
of the noise multiplier $\nu$. Further, in order to decrease variability of the
results, a total of $10$ runs have been performed with different seeds for every 
combination of number of hypotheses and noise multiplier. Results are shown in Figure
\ref{fig:hospital}. 

When the federated training is performed with only $1$ initial hypothesis,
the accuracy of the model is poor, which is indicative of the model not being able to capture
the variety of data distributions that is being fed with. In fact, increasing to $3$ the number of
initial hypotheses for the parameter vector leads to the biggest improvement on the RMSE loss.
Additionally, we can see that the model's performance degrades with increasing values of the
noise multiplier (and therefore increasing $\varepsilon$'s), as expected. The large variability in
performance when the communication is sanitized with $\nu \in \left\{2,3,5\right\}$ may be due
to the assumption in Equation \eqref{condition1} failing to be satisfied in certain runs, 
leading to all clients being grouped under a single cluster, and reaching RMSE
comparable to that obtained with only $1$ initial hypothesis. The best results in terms of
both accuracy and low variability are when the number of initial hypotheses is set to $5$ and 
$7$. Although a prescriptive characterization of the decrease in model's performance with varying
noise multiplier levels is yet to be derived, we highlight 
how experimentally there are regions of the 
hyper-parameter space (i.e. the choice of $\nu$ and the number of initial hypotheses) where a
reasonable compromise can be found between privacy and model personalization.

Finding the  privacy leakage is straight-forward, as each time a user is required to participate in 
a training round it will enjoy $\varepsilon \Vert \delta_c^{(t)} \Vert_2 = n/\nu = 11/\nu$ differential
privacy with any point in its $\delta_c^{(t)} $-neighborhood. Accordingly, 
Figure \ref{priv-leak-hospital} provides the empirical privacy leakage distribution of the clients involved in a particular training configuration, whereas Table \ref{table:privacy-budget-correct} shows privacy leakage statics over multiple rounds and for all configurations.

\begin{table}
    \begin{center}
        \begin{tabular}{*5c}
            \toprule
            & \multicolumn{4}{c}{Hypotheses}  \\
            \midrule
            Noise Multiplier & 7 & 5 & 3 & 1 \\
            \hline
            0 & -, - & -, - & -, - & -, - \\
            \hline
            0.100 & 517.0, 1551.0 & 418.0, 1342.0 & 473.0, 1386.0 & 528.0, 1540.0 \\
            \hline
            1 & 36.3, 126.5 & 40.7, 127.6 & 44.0, 138.6 & 49.5, 147.4 \\
            \hline
            2 & 15.4, 57.8 & 14.3, 54.5 & 22.0, 69.3 & 21.5, 66.6 \\
            \hline
            3 & 7.7, 32.3 & 8.4, 36.7 & 12.5, 40.0 & 12.1, 40.0 \\
            \hline
            5 & 5.7, 21.3 & 5.9, 22.0 & 5.5, 21.6 & 5.3, 20.9 \\
            \bottomrule
        \end{tabular}
    \end{center}
    \caption{Regarding the experiment on hospital charge data, for every combination of Noise Multiplier $\times$ Number of Hypotheses, the median and maximum local privacy budgets are reported, over the whole set of clients. These values are averaged over $10$ runs with different seeds. $\nu=0$ means no privacy guarantee and infinite privacy leakage.}
    \label{table:privacy-budget-correct}
\end{table}

\begin{figure}[h]
\centering
\includegraphics[scale=0.65]{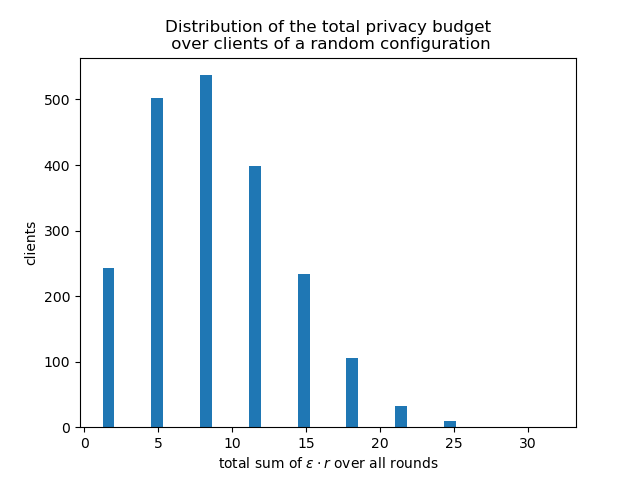}
\caption{For the experiment on hospital charge data, this histogram plots the empirical distribution of the privacy budget over the clients in a particular configuration: $\nu=3$, 5 initial hypotheses, seed $=3$, 
$r$ is the radius of the neighborhood, the total number of clients is 2062.}
\label{priv-leak-hospital}
\end{figure}

\begin{figure}[h]
    \centering
    \includegraphics[scale=0.65]{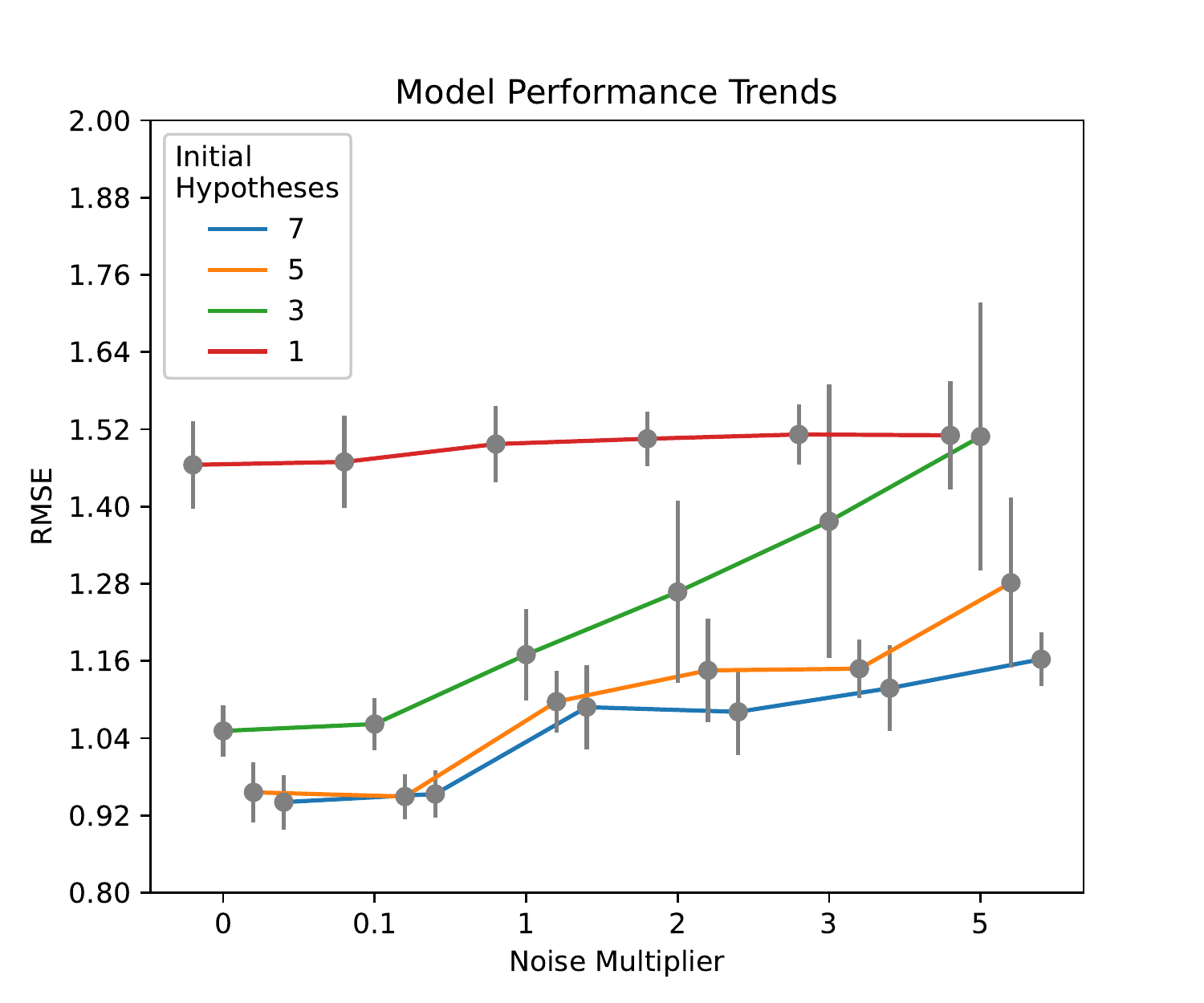}
    \caption{RMSE values for models trained with Algorithm \ref{alg:pifca} on the Hospital 
    Charge Dataset. Error bars show $\pm \sigma$, with $\sigma$ the empirical standard
    deviation. Lower RMSE values are better for accuracy.}
    \label{fig:hospital}
\end{figure}

\subsection{FEMNIST image classification} \label{nn}
In this Section we evaluate how Algorithm \ref{alg:pifca} behaves when tested beyond the scope
of its applicability, as described in Section \ref{sec:limitations}. 
The task consists in performing image classification on the
FEMNIST \cite{caldas2018leaf} dataset, which is a standard benchmark dataset for federated
learning, based on EMNIST  \cite{cohen2017emnist} and with the data points
grouped by user. It consists of a large number of images of handwritten digits, lower and upper
case letters of the latin alphabet. As a pre-processing step,
images of client $c$ are rotated $90$ degrees counter-clockwise depending on the realization of
the random variable $\text{rot}_c \sim \text{Bernoulli}(0.5)$. This is a common practice in
machine learning  to simulate local datasets held by different 
clients being generated by very different distributions 
\cite{ghosh, goodfellow2013empirical, kirkpatrick2017overcoming, lopez2017gradient}.

The chosen architecture is described in Table \ref{tab:arch} and yields a parameter vector 
$\theta \in \mathbb{R}^{n_0}$, $n_0 = 1206590$.
Runs are performed with a maximum of $500$ rounds of
federated optimization, unless $5$ consecutive validation rounds are conducted without
improvements on the validation loss. The latter is evaluated on a held out set of clients,
consisting of $10\%$ of the total number. Validation is performed every $5$ training rounds,
thus the process terminates after $25$ rounds without model's performance improvement.
The optimization process aims to minimize either the RMSE loss or the 
Cross Entropy loss \cite{zhang2018crossentropy} (to further depart from earlier assumptions)
between model's predictions and the target class. Results are presented in Table \ref{nn-results}.
For Cross Entropy, we see a wide range of $\nu$ values with comparable average accuracy. In
particular, the best performing model is being trained with a non-zero noise multiplier,
which may be explained by a regularizing effect of the additive noise. This is especially true
for the RMSE loss, where the best performing model is trained with $\nu = 3$. For all the runs,
we highlight a generally low standard deviation in the results.

Note that with the choice of the range of noise multipliers $\nu$ the corresponding value for the
privacy leakage $\varepsilon \Vert \delta_c^{(t)} \Vert_2 = n/\nu = n_0/\nu$ wold be enormous,
and would not provide any meaningful guarantee, in theory. As already mentioned in Section 
\ref{sec:limitations}, that is true as long as we want to use the Laplace mechanism to be
effective against \emph{any} adversary. Still, it is possible to validate, in practice, whether
it can protect against a \emph{specific} attack: DLG \cite{deep-leakage}. The threat model for
this attack is very fitting for a federated learning scenario. In brief:
an honest-but-curious server communicates to a set of clients the parameter vector
$\theta_{\bar{j}}^{(t)}$ (among the other $k-1$ hypotheses) at iteration $t$ and 
receives the updated model parameters $\theta_{\bar{j}, c}^{(t)}$ 
from client $c$. The server can easily retrieve the true
parameter update $\delta_c^{(t)} = \theta_{\bar{j}, c}^{(t)} - \theta_{\bar{j}}^{(t)} $
if no sanitization is performed. Under the assumption that the client
performs one single optimization step, this results in being the gradient
scaled down by the local step size. The server then tries to recreate the input samples that
generated such gradient. The process of gradient matching can be cast into a nonlinear
minimization problem and be solved itself by gradient descent.

If sanitization is performed, the server is left with matching a corrupted gradient.
In \cite{deep-leakage} the authors evaluate disturbing the gradient with Gaussian and Laplace 
(with $L_1$ distance) noise as a privacy mechanism. In the following, we evaluate if the
distribution of the Laplace mechanism (under $L_2$ distance)
in Lemma \ref{lemma_laplace} is effective in protecting from the DLG attack. 
In order to be on the safe side, tests were conducted with the best 
possible conditions for the attacker:
a modified model architecture, so that DLG conditions are met (e.g. all activation functions
are replaced with the sigmoid non-linearity to have a twice-differentiable model); batch
size reduced to $1$, as the gradient matching optimization problem is easier to solve in this
setting; and a single local optimization step. Since the gradient can vary widely for parameters 
in different layers of the neural network, we apply the Laplace mechanism independently on the
parameter vector of each NN layer, and communicate to the server the collection of sanitized
parameter vectors. The practice of sanitizing each layer independently has already been
effectively evaluated in \cite{liu2020padl}.

In Figure \ref{fig:dlg} results are reported for application of
the noise multiplier values adopted also in Table \ref{nn-results}. When $\nu = 10^{-3}$ the
ground truth image is fully reconstructed. Up to $\nu = 10^{-1}$ we see that at least
partial reconstruction is possible. Finally, for $\nu \geq 1$ we see that, experimentally,
the DLG attack fails to reconstruct input samples.

\begin{table}[htbp] 
\centering
\begin{tabular}{ccccc}
\hline
                 & \multicolumn{2}{c}{Cross Entropy loss}
                 & \multicolumn{2}{c}{RMSE loss}                                                                                                  \\ \hline
\begin{tabular}[c]{@{}c@{}}Noise\\ Multiplier\end{tabular}  & \begin{tabular}[c]{@{}c@{}}Average \\ Accuracy\end{tabular} & \begin{tabular}[c]{@{}c@{}}Standard\\ Deviation\end{tabular} & \begin{tabular}[c]{@{}c@{}}Average\\ Accuracy\end{tabular} & \begin{tabular}[c]{@{}c@{}}Standard\\ Deviation\end{tabular} \\ \hline
0                & 0.832                                                       & $\pm$ 0.012                                                        & 0.801                                                      & $\pm$ 0.001                                                        \\
0.001            & 0.843                                                       & $\pm$ 0.006                                                        & 0.813                                                      & $\pm$ 0.014                                                        \\
0.01             & 0.832                                                       & $\pm$ 0.017                                                        & 0.805                                                      & $\pm$ 0.008                                                        \\
0.1              & 0.834                                                       & $\pm$ 0.026                                                        & 0.808                                                      & $\pm$ 0.019                                                        \\
1                & 0.834                                                       & $\pm$ 0.014                                                        & 0.814                                                      & $\pm$ 0.012                                                        \\
3                & 0.835                                                       & $\pm$ 0.017                                                        & 0.825                                                      & $\pm$ 0.010                                                        \\
5                & 0.812                                                       & $\pm$ 0.016                                                        & 0.787                                                      & $\pm$ 0.003                                                        \\
10               & 0.692                                                       & $\pm$ 0.002                                                        & 0.687                                                      & $\pm$ 0.014                                                        \\
15               & 0.561                                                       & $\pm$ 0.005                                                        & 0.622                                                      & $\pm$ 0.003                                                        \\ \hline
\end{tabular}
\caption{Average classification accuracy and standard deviation of a convolutional neural
network over three runs seeded with different values. Experiments tested the effect of increasing noise values on the
validation accuracy.} \label{nn-results}
\end{table}

\begin{figure}[htbp]
  \centering
  \includegraphics[scale=0.65]{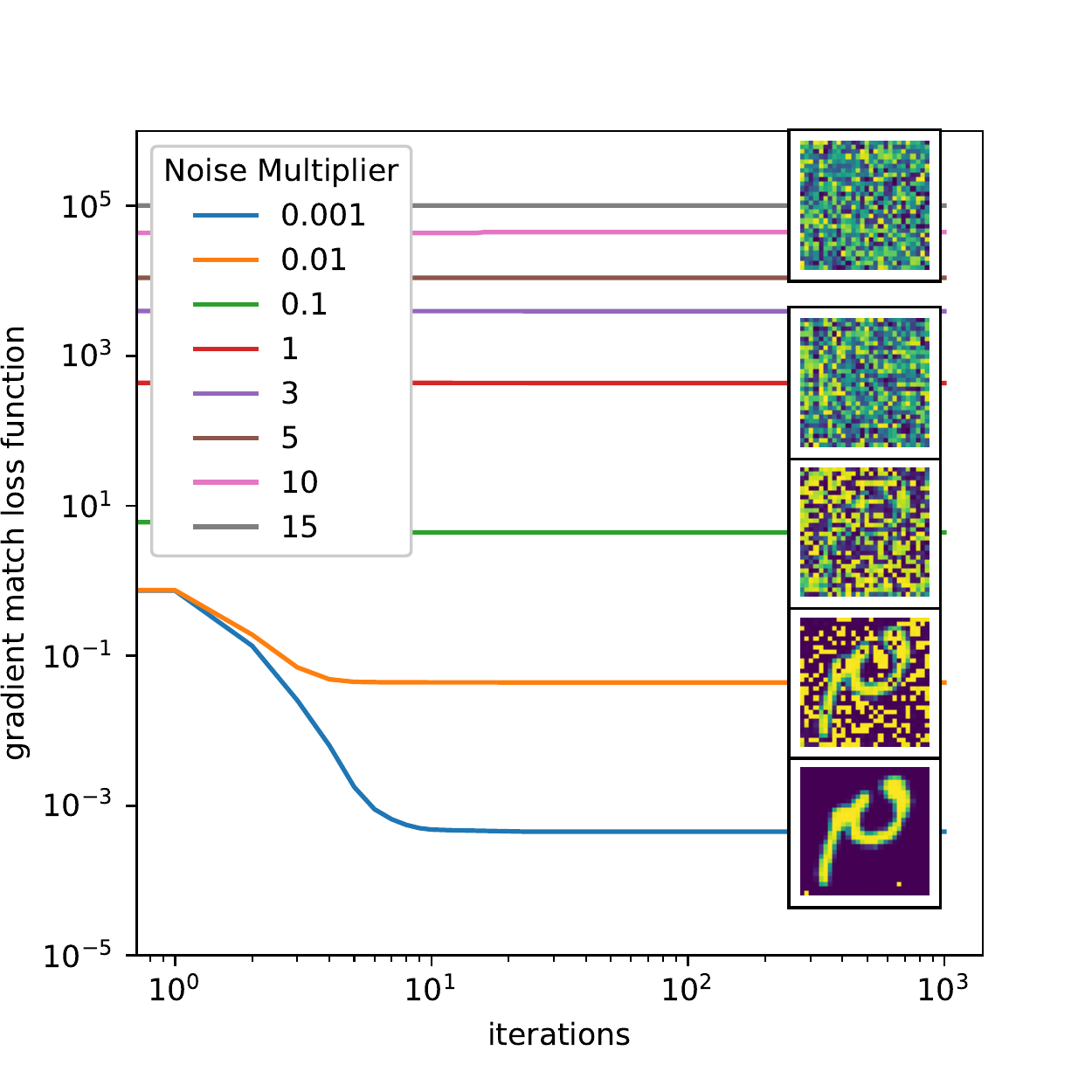}
  \caption{Effects of the Laplace mechanism in Lemma \ref{lemma_laplace} with
  different noise multipliers (ref) as a defense strategy against the DLG
attack.}\label{fig:dlg}
\end{figure}

\begin{table}[htbp]
\centering
\begin{tabular}{cc}
\hline
Layer & Properties \\ \hline
2D Convolution & \begin{tabular}[c]{@{}c@{}}kernel size: (2,2)\\ stride: (1,1)\\ nonlinearity: ReLU\\ output features: 32\end{tabular} \\ \hline
2D Convolution & \begin{tabular}[c]{@{}c@{}}kernel size: (2,2)\\ stride: (1,1)\\ nonlinearity: ReLU\\ output features: 64\end{tabular} \\ \hline
2D Max Pool & \begin{tabular}[c]{@{}c@{}}kernel size: (2,2)\\ stride: (2,2)\\ nonlinearity: ReLU\end{tabular} \\ \hline
Fully Connected & \begin{tabular}[c]{@{}c@{}}nonlinearity: ReLU\\ units: 128\end{tabular} \\ \hline
Fully Connected & \begin{tabular}[c]{@{}c@{}}nonlinearity: ReLU\\ units: 62\end{tabular} \\ \hline
\end{tabular}\caption{NN architecture adopted in the experiments of Section \ref{nn}} \label{tab:arch}
\end{table}

\section{Conclusion}\label{section:Conclusion}
This work presents the challenging task of optimizing federated learning models over the three 
dimensions of privacy, accuracy and personalization. The problem of 
preserving the privacy of individuals is treated under the framework of $d$-privacy, which 
provides guarantees of indistinguishability that depend on the distance between any two points.
Here, such points lie in the parameter space of machine learning models, which are sanitized and  
communicated to a central server for aggregation, in order to get closer to the
optimal parameters iteratively. Given that the data distribution among individuals is unknown,
it is reasonable to assume a mixture of multiple distributions.
Clustering the sanitized parameter vectors 
released by the clients with the $k$-means algorithm shows to be a good proxy for aggregating
clients with similar data distributions. This is possible because $d$-private mechanisms
preserve the topology of the domain of true values. To that end, the Laplace mechanism under
Euclidean distance was defined, together with a procedure for sampling from its distribution.
Experimental results validate our claims and the limitations of the theory developed here are
discussed. In particular, our privacy preserving mechanism shows to be promising when machine
learning models have a \emph{small} number of parameters. Although formal privacy guarantees degrade sharply with large machine learning models, we show experimentally that the Laplace mechanism under Euclidean distance is effective at least against client's data reconstruction by DLG attack.

As future work, we want to explore other privacy mechanisms, which may be  more effective in providing a good trade-off between privacy and accuracy in the context of machine learning.  
Furthermore, we are interested in studying more complex federated learning scenarios where participants and datasets may change over time.

\section*{Acknowledgment}
The work of Sayan Biswas, Kangsoo Jung, and Catuscia Palamidessi was supported by the European Research Council (ERC) project HYPATIA under the European Union’s Horizon 2020 research and innovation programme. Grant agreement no. 835294.

\bibliographystyle{splncs04}
\bibliography{biblio}

\begin{thebibliography}{10}
\providecommand{\url}[1]{\texttt{#1}}
\providecommand{\urlprefix}{URL }
\providecommand{\doi}[1]{https://doi.org/#1}

\bibitem{abadi}
Abadi, M., Chu, A., Goodfellow, I., McMahan, H.B., Mironov, I., Talwar, K.,
  Zhang, L.: Deep learning with differential privacy. In: Proceedings of the
  2016 ACM SIGSAC conference on computer and communications security. pp.
  308--318 (2016)

\bibitem{cpsgd}
Agarwal, N., Suresh, A.T., Yu, F.X.X., Kumar, S., McMahan, B.: cpsgd:
  Communication-efficient and differentially-private distributed sgd. Advances
  in Neural Information Processing Systems  \textbf{31} (2018)

\bibitem{geo}
Andr{\'e}s, M.E., Bordenabe, N.E., Chatzikokolakis, K., Palamidessi, C.:
  Geo-indistinguishability: Differential privacy for location-based systems.
  In: Proceedings of the 2013 ACM SIGSAC conference on Computer \&
  communications security. pp. 901--914 (2013)

\bibitem{adaptive}
Andrew, G., Thakkar, O., McMahan, B., Ramaswamy, S.: Differentially private
  learning with adaptive clipping. Advances in Neural Information Processing
  Systems  \textbf{34} (2021)

\bibitem{balle2019privacy}
Balle, B., Bell, J., Gasc{\'o}n, A., Nissim, K.: The privacy blanket of the
  shuffle model. In: Annual International Cryptology Conference. pp. 638--667.
  Springer (2019)

\bibitem{balle2020private}
Balle, B., Bell, J., Gasc{\'o}n, A., Nissim, K.: Private summation in the
  multi-message shuffle model. In: Proceedings of the 2020 ACM SIGSAC
  Conference on Computer and Communications Security. pp. 657--676 (2020)

\bibitem{ldp-bounds}
Bassily, R., Nissim, K., Stemmer, U., Guha~Thakurta, A.: Practical locally
  private heavy hitters. In: Guyon, I., Luxburg, U.V., Bengio, S., Wallach, H.,
  Fergus, R., Vishwanathan, S., Garnett, R. (eds.) Advances in Neural
  Information Processing Systems. vol.~30. Curran Associates, Inc. (2017),
  \url{https://proceedings.neurips.cc/paper/2017/file/3d779cae2d46cf6a8a99a35ba4167977-Paper.pdf}

\bibitem{bittau2017prochlo}
Bittau, A., Erlingsson, {\'U}., Maniatis, P., Mironov, I., Raghunathan, A.,
  Lie, D., Rudominer, M., Kode, U., Tinnes, J., Seefeld, B.: Prochlo: Strong
  privacy for analytics in the crowd. In: Proceedings of the 26th Symposium on
  Operating Systems Principles. pp. 441--459 (2017)

\bibitem{secure-aggregation}
Bonawitz, K., Ivanov, V., Kreuter, B., Marcedone, A., McMahan, H.B., Patel, S.,
  Ramage, D., Segal, A., Seth, K.: Practical secure aggregation for federated
  learning on user-held data. arXiv preprint arXiv:1611.04482  (2016)

\bibitem{Bordenabe:14:CCS}
Bordenabe, N.E., Chatzikokolakis, K., Palamidessi, C.: Optimal
  geo-indistinguishable mechanisms for location privacy. In: Proceedings of the
  21th ACM Conference on Computer and Communications Security (CCS 2014) (2014)

\bibitem{bottou2012stochastic}
Bottou, L.: Stochastic gradient descent tricks. In: Neural networks: Tricks of
  the trade, pp. 421--436. Springer (2012)

\bibitem{caldas2018leaf}
Caldas, S., Duddu, S.M.K., Wu, P., Li, T., Kone{\v{c}}n{\`y}, J., McMahan,
  H.B., Smith, V., Talwalkar, A.: Leaf: A benchmark for federated settings.
  arXiv preprint arXiv:1812.01097  (2018)

\bibitem{broadening}
Chatzikokolakis, K., Andr{\'e}s, M.E., Bordenabe, N.E., Palamidessi, C.:
  Broadening the scope of differential privacy using metrics. In: International
  Symposium on Privacy Enhancing Technologies Symposium. pp. 82--102. Springer
  (2013)

\bibitem{cheu2019distributed}
Cheu, A., Smith, A., Ullman, J., Zeber, D., Zhilyaev, M.: Distributed
  differential privacy via shuffling. In: Annual International Conference on
  the Theory and Applications of Cryptographic Techniques. pp. 375--403.
  Springer (2019)

\bibitem{cheu2021differentially}
Cheu, A., Zhilyaev, M.: Differentially private histograms in the shuffle model
  from fake users. arXiv preprint arXiv:2104.02739  (2021)

\bibitem{cohen2017emnist}
Cohen, G., Afshar, S., Tapson, J., Van~Schaik, A.: Emnist: Extending mnist to
  handwritten letters. In: 2017 international joint conference on neural
  networks (IJCNN). pp. 2921--2926. IEEE (2017)

\bibitem{DuchiLDP}
Duchi, J.C., Jordan, M.I., Wainwright, M.J.: Local privacy and statistical
  minimax rates. In: 2013 IEEE 54th Annual Symposium on Foundations of Computer
  Science. pp. 429--438 (2013). \doi{10.1109/FOCS.2013.53}

\bibitem{DworkDP2}
Dwork, C., Kenthapadi, K., McSherry, F., Mironov, I., Naor, M.: Our data,
  ourselves: Privacy via distributed noise generation. In: Vaudenay, S. (ed.)
  Advances in Cryptology - EUROCRYPT 2006. pp. 486--503. Springer Berlin
  Heidelberg, Berlin, Heidelberg (2006)

\bibitem{DworkDP1}
Dwork, C., McSherry, F., Nissim, K., Smith, A.: Calibrating noise to
  sensitivity in private data analysis. In: Halevi, S., Rabin, T. (eds.) Theory
  of Cryptography. pp. 265--284. Springer Berlin Heidelberg, Berlin, Heidelberg
  (2006)

\bibitem{DworkDP_Compositionality}
Dwork, C., Roth, A.: The algorithmic foundations of differential privacy.
  Found. Trends Theor. Comput. Sci.  \textbf{9}(3–4),  211–407 (aug 2014).
  \doi{10.1561/0400000042}, \url{https://doi.org/10.1561/0400000042}

\bibitem{foundations-of-dp}
Dwork, C., Roth, A., et~al.: The algorithmic foundations of differential
  privacy. Found. Trends Theor. Comput. Sci.  \textbf{9}(3-4),  211--407 (2014)

\bibitem{erlingsson2020encode}
Erlingsson, {\'U}., Feldman, V., Mironov, I., Raghunathan, A., Song, S.,
  Talwar, K., Thakurta, A.: Encode, shuffle, analyze privacy revisited:
  Formalizations and empirical evaluation. arXiv preprint arXiv:2001.03618
  (2020)

\bibitem{erlingsson2019amplification}
Erlingsson, {\'U}., Feldman, V., Mironov, I., Raghunathan, A., Talwar, K.,
  Thakurta, A.: Amplification by shuffling: From local to central differential
  privacy via anonymity. In: Proceedings of the Thirtieth Annual ACM-SIAM
  Symposium on Discrete Algorithms. pp. 2468--2479. SIAM (2019)

\bibitem{feldman2020hiding}
Feldman, V., McMillan, A., Talwar, K.: Hiding among the clones: A simple and
  nearly optimal analysis of privacy amplification by shuffling. arXiv preprint
  arXiv:2012.12803  (2020)

\bibitem{Fernandes:21:LICS}
Fernandes, N., McIver, A., Morgan, C.: The laplace mechanism has optimal
  utility for differential privacy over continuous queries. In: 36th Annual
  {ACM/IEEE} Symposium on Logic in Computer Science, {LICS} 2021. pp. 1--12.
  {IEEE} (2021). \doi{10.1109/LICS52264.2021.9470718},
  \url{https://doi.org/10.1109/LICS52264.2021.9470718}

\bibitem{flcdp}
Geyer, R.C., Klein, T., Nabi, M.: Differentially private federated learning: A
  client level perspective. arXiv preprint arXiv:1712.07557  (2017)

\bibitem{ghosh}
Ghosh, A., Chung, J., Yin, D., Ramchandran, K.: An efficient framework for
  clustered federated learning. Advances in Neural Information Processing
  Systems  \textbf{33},  19586--19597 (2020)

\bibitem{shuffled}
Girgis, A., Data, D., Diggavi, S., Kairouz, P., Suresh, A.T.: Shuffled model of
  differential privacy in federated learning. In: International Conference on
  Artificial Intelligence and Statistics. pp. 2521--2529. PMLR (2021)

\bibitem{goodfellow2013empirical}
Goodfellow, I.J., Mirza, M., Xiao, D., Courville, A., Bengio, Y.: An empirical
  investigation of catastrophic forgetting in gradient-based neural networks.
  arXiv preprint arXiv:1312.6211  (2013)

\bibitem{fed-l-3}
Hard, A., Rao, K., Mathews, R., Ramaswamy, S., Beaufays, F., Augenstein, S.,
  Eichner, H., Kiddon, C., Ramage, D.: Federated learning for mobile keyboard
  prediction. arXiv preprint arXiv:1811.03604  (2018)

\bibitem{hitaj2017deep}
Hitaj, B., Ateniese, G., Perez-Cruz, F.: Deep models under the gan: information
  leakage from collaborative deep learning. In: Proceedings of the 2017 ACM
  SIGSAC conference on computer and communications security. pp. 603--618
  (2017)

\bibitem{huetal}
Hu, R., Guo, Y., Li, H., Pei, Q., Gong, Y.: Personalized federated learning
  with differential privacy. IEEE Internet of Things Journal  \textbf{7}(10),
  9530--9539 (2020)

\bibitem{kairouz2016discrete}
Kairouz, P., Bonawitz, K., Ramage, D.: Discrete distribution estimation under
  local privacy. In: International Conference on Machine Learning. pp.
  2436--2444. PMLR (2016)

\bibitem{what-can-we}
Kasiviswanathan, S.P., Lee, H.K., Nissim, K., Raskhodnikova, S., Smith, A.:
  What can we learn privately? SIAM Journal on Computing  \textbf{40}(3),
  793--826 (2011)

\bibitem{kirkpatrick2017overcoming}
Kirkpatrick, J., Pascanu, R., Rabinowitz, N., Veness, J., Desjardins, G., Rusu,
  A.A., Milan, K., Quan, J., Ramalho, T., Grabska-Barwinska, A., et~al.:
  Overcoming catastrophic forgetting in neural networks. Proceedings of the
  national academy of sciences  \textbf{114}(13),  3521--3526 (2017)

\bibitem{fed-l-1}
Kone{\v{c}}ny, J., McMahan, H.B., Ramage, D., Richt{\'a}rik, P.: Federated
  optimization: Distributed machine learning for on-device intelligence. arXiv
  preprint arXiv:1610.02527  (2016)

\bibitem{fed-l-2}
Kone{\v{c}}ny, J., McMahan, H.B., Yu, F.X., Richt{\'a}rik, P., Suresh, A.T.,
  Bacon, D.: Federated learning: Strategies for improving communication
  efficiency. arXiv preprint arXiv:1610.05492  (2016)

\bibitem{koskela2021tight}
Koskela, A., Heikkil{\"a}, M.A., Honkela, A.: Tight accounting in the shuffle
  model of differential privacy. arXiv preprint arXiv:2106.00477  (2021)

\bibitem{koskela2021tightdiscrete}
Koskela, A., J{\"a}lk{\"o}, J., Prediger, L., Honkela, A.: Tight differential
  privacy for discrete-valued mechanisms and for the subsampled gaussian
  mechanism using fft. In: International Conference on Artificial Intelligence
  and Statistics. pp. 3358--3366. PMLR (2021)

\bibitem{lemetayer:hal-01420983}
Le~M{\'e}tayer, D., De, S.J.: {PRIAM: a Privacy Risk Analysis Methodology}. In:
  Livraga, G., Torra, V., Aldini, A., Martinelli, F., Suri, N. (eds.) {Data
  Privacy Management and Security Assurance}. {Springer}, Heraklion, Greece
  (Sep 2016), \url{https://hal.inria.fr/hal-01420983}

\bibitem{li2020federated}
Li, T., Sahu, A.K., Talwalkar, A., Smith, V.: Federated learning: Challenges,
  methods, and future directions. IEEE Signal Processing Magazine
  \textbf{37}(3),  50--60 (2020)

\bibitem{liu2020padl}
Liu, X., Li, H., Xu, G., Liu, S., Liu, Z., Lu, R.: Padl: Privacy-aware and
  asynchronous deep learning for iot applications. IEEE Internet of Things
  Journal  \textbf{7}(8),  6955--6969 (2020)

\bibitem{lopez2017gradient}
Lopez-Paz, D., Ranzato, M.: Gradient episodic memory for continual learning.
  Advances in neural information processing systems  \textbf{30} (2017)

\bibitem{mansour}
Mansour, Y., Mohri, M., Ro, J., Suresh, A.T.: Three approaches for
  personalization with applications to federated learning. arXiv preprint
  arXiv:2002.10619  (2020)

\bibitem{fed-l-0}
McMahan, B., Moore, E., Ramage, D., Hampson, S., y~Arcas, B.A.:
  Communication-efficient learning of deep networks from decentralized data.
  In: Artificial intelligence and statistics. pp. 1273--1282. PMLR (2017)

\bibitem{dprnn}
McMahan, H.B., Ramage, D., Talwar, K., Zhang, L.: Learning differentially
  private recurrent language models. arXiv preprint arXiv:1710.06963  (2017)

\bibitem{meehan2021shuffling}
Meehan, C., Chowdhury, A.R., Chaudhuri, K., Jha, S.: A shuffling framework for
  local differential privacy. arXiv preprint arXiv:2106.06603  (2021)

\bibitem{nasr2019comprehensive}
Nasr, M., Shokri, R., Houmansadr, A.: Comprehensive privacy analysis of deep
  learning: Passive and active white-box inference attacks against centralized
  and federated learning. In: 2019 IEEE symposium on security and privacy (SP).
  pp. 739--753. IEEE (2019)

\bibitem{nist}
NIST: Nist privacy framework core,
  \url{https://www.nist.gov/system/files/documents/2021/05/05/NIST-Privacy-Framework-V1.0-Core-PDF.pdf}

\bibitem{gen-fed-avg}
Reddi, S., Charles, Z., Zaheer, M., Garrett, Z., Rush, K., Kone{\v{c}}ny, J.,
  Kumar, S., McMahan, H.B.: Adaptive federated optimization. arXiv preprint
  arXiv:2003.00295  (2020)

\bibitem{sattler}
Sattler, F., M{\"u}ller, K.R., Samek, W.: Clustered federated learning:
  Model-agnostic distributed multitask optimization under privacy constraints.
  IEEE transactions on neural networks and learning systems  \textbf{32}(8),
  3710--3722 (2020)

\bibitem{shokri}
Shokri, R., Shmatikov, V.: Privacy-preserving deep learning. In: Proceedings of
  the 22nd ACM SIGSAC conference on computer and communications security. pp.
  1310--1321 (2015)

\bibitem{sommer2019privacy}
Sommer, D.M., Meiser, S., Mohammadi, E.: Privacy loss classes: The central
  limit theorem in differential privacy. Proceedings on privacy enhancing
  technologies  \textbf{2019}(2),  245--269 (2019)

\bibitem{suriyakumar2021chasing}
Suriyakumar, V.M., Papernot, N., Goldenberg, A., Ghassemi, M.: Chasing your
  long tails: Differentially private prediction in health care settings. In:
  Proceedings of the 2021 ACM Conference on Fairness, Accountability, and
  Transparency. pp. 723--734 (2021)

\bibitem{ldpfl}
Truex, S., Liu, L., Chow, K.H., Gursoy, M.E., Wei, W.: Ldp-fed: Federated
  learning with local differential privacy. In: Proceedings of the Third ACM
  International Workshop on Edge Systems, Analytics and Networking. pp. 61--66
  (2020)

\bibitem{zhang2018crossentropy}
Zhang, Z., Sabuncu, M.: Generalized cross entropy loss for training deep neural
  networks with noisy labels. Advances in neural information processing systems
   \textbf{31} (2018)

\bibitem{zhao2020local}
Zhao, Y., Zhao, J., Yang, M., Wang, T., Wang, N., Lyu, L., Niyato, D., Lam,
  K.Y.: Local differential privacy-based federated learning for internet of
  things. IEEE Internet of Things Journal  \textbf{8}(11),  8836--8853 (2020)

\bibitem{deep-leakage}
Zhu, L., Liu, Z., Han, S.: Deep leakage from gradients. Advances in Neural
  Information Processing Systems  \textbf{32} (2019)

\end{thebibliography}

\newpage

\end{document}